\documentclass[11pt, letterpaper]{article} %
\usepackage{times}
\usepackage{url}            %
\usepackage{booktabs}       %
\usepackage{amsfonts}       %
\usepackage{nicefrac}       %
\usepackage{xspace}
\usepackage{multirow}
\usepackage{amsmath}
\usepackage{algorithm}
\usepackage{algorithmic}
\usepackage{color, colortbl}
\usepackage{enumitem}
\usepackage{comment}
\usepackage{bm}
\usepackage{fullpage}
\usepackage{caption}
\usepackage{subcaption}
\usepackage{amssymb}
\usepackage{mathtools}
\usepackage{amsthm}
\usepackage[dvipsnames]{xcolor}
\usepackage[colorlinks=true, linkcolor=red, citecolor=blue, urlcolor=WildStrawberry]{hyperref}
\usepackage[capitalize,noabbrev]{cleveref}
\usepackage{natbib}
\usepackage{mathrsfs}
\usepackage{parskip}
\usepackage{amsxtra}
\usepackage{mathtools}
\usepackage{bbm}
\usepackage{enumitem}

\allowdisplaybreaks

\DeclarePairedDelimiter{\crl}{\{}{\}}

\DeclarePairedDelimiter{\ceil}{\lceil}{\rceil}
\DeclarePairedDelimiter{\floor}{\lfloor}{\rfloor}

\DeclareMathOperator*{\argmin}{argmin} 
\DeclareMathOperator*{\argmax}{argmax}

\newcommand{\ls}{\ell}
\newcommand{\Out}{\mathcal{Y}}
\newcommand{\out}{y}

\newcommand{\Eluder}{\mathcal{E}}
\newcommand{\signal}{p_{\mathrm{signal}}}

\newcommand{\ind}[1]{\mathbbm{1}\crl*{#1}}    

\newcommand{\reals}{\mathbb{R}}
\newcommand{\action}{a}
\newcommand{\Action}{\mathcal{A}}
\newcommand{\Optimistic}{O}
\newcommand{\Pessimistic}{P}

\newcommand{\cont}{x}
\newcommand{\Context}{\mathcal{X}}
\newcommand{\Map}{\mathbf{M}}

\newcommand{\Regsq}{\mathrm{Reg}_{\mathrm{OR}}}
\newcommand{\regretol}{\mathrm{Reg}_{\mathrm{OL}}}

\newcommand{\f}{f}
\newcommand{\F}{\mathcal F}
\newcommand{\bigO}{\mathcal O}
\newcommand\smallO{o}
\newcommand{\lif}{\f}

\newcommand{\G}{\mathcal G}

\newcommand{\zf}{z}
\newcommand{\hzf}{\hat{z}}
\newcommand{\eg}{\xi}

\newcommand{\Link}{\Phi}
\newcommand{\link}{\phi}
\newcommand{\lell}{\ell_\link}

\newcommand{\x}{\action}
\newcommand{\X}{\mathcal \Action}

\newcommand{\regret}{\text{Regret}}

\newcommand{\conv}{\text{Conv}}

\newcommand{\oraclesq}{\mathsf{Oracle}_{\mathsf{OR}}}
\newcommand{\oracleol}{\mathsf{Oracle}_{\mathsf{OL}}}

\DeclareMathOperator*{\E}{{}\mathbb E}
\newcommand{\Ex}[2]{\E_{#1} \left[ #2 \right]}
\newcommand{\Simplex}{\Omega}

\def\ddefloop#1{\ifx\ddefloop#1\else\ddef{#1}\expandafter\ddefloop\fi}
\def\ddef#1{\expandafter\def\csname bb#1\endcsname{\ensuremath{\mathbb{#1}}}}
\ddefloop ABCDEFGHIJKLMNOPQRSTUVWXYZ\ddefloop
\def\ddefloop#1{\ifx\ddefloop#1\else\ddef{#1}\expandafter\ddefloop\fi}
\def\ddef#1{\expandafter\def\csname b#1\endcsname{\ensuremath{\mathbf{#1}}}}
\ddefloop ABCDEFGHIJKLMNOPQRSTUVWXYZ\ddefloop
\def\ddef#1{\expandafter\def\csname c#1\endcsname{\ensuremath{\mathcal{#1}}}}
\ddefloop ABCDEFGHIJKLMNOPQRSTUVWXYZ\ddefloop
\def\ddef#1{\expandafter\def\csname h#1\endcsname{\ensuremath{\widehat{#1}}}}
\ddefloop ABCDEFGHIJKLMNOPQRSTUVWXYZabcdefghijklmnopqrsuvwxyz\ddefloop    
\def\ddef#1{\expandafter\def\csname hc#1\endcsname{\ensuremath{\widehat{\mathcal{#1}}}}}
\ddefloop ABCDEFGHIJKLMNOPQRSTUVWXYZ\ddefloop
\def\ddef#1{\expandafter\def\csname t#1\endcsname{\ensuremath{\widetilde{#1}}}}
\ddefloop ABCDEFGHIJKLMNOPQRSTUVWXYZ\ddefloop
\def\ddef#1{\expandafter\def\csname tc#1\endcsname{\ensuremath{\widetilde{\mathcal{#1}}}}}
\ddefloop ABCDEFGHIJKLMNOPQRSTUVWXYZ\ddefloop

\usepackage{tikz}

\newtheorem*{proposition*}{Proposition}
\newtheorem*{lemma*}{Lemma}
\newtheorem*{theorem*}{Theorem}
\newtheorem*{corollary*}{Corollary}

\newcommand{ \edim }{\Eluder}
\usepackage{color-edits} 
\addauthor{ky}{red}  
\addauthor{wy}{orange}
\usepackage{algorithmic}
\usepackage{float}

\theoremstyle{plain}
\newtheorem{theorem}{Theorem}[section]
\newtheorem{lemma}[theorem]{Lemma}
\newtheorem{corollary}[theorem]{Corollary}
\newtheorem{proposition}[theorem]{Proposition}

\theoremstyle{definition}
\newtheorem{assumption}{Assumption}[section]
\newtheorem{definition}{Definition}[section]

\newtheorem*{remark}{Remark}

\title{Online Learning with Unknown Constraints\thanks{Authors are listed in alphabetical order of their last names.}}
\author{Karthik Sridharan \\ Cornell University
\and Seung Won Wilson Yoo \thanks{Emails: \texttt {\{ks999, sy536\}@cornell.edu}} \\ Cornell University 
}
\date{}

\begin{document}
\maketitle

\begin{abstract}%
We consider the problem of online learning where the sequence of actions played by the learner must adhere to an unknown safety constraint at every round. The goal is to minimize regret with respect to the best safe action in hindsight while simultaneously satisfying the safety constraint with high probability on each round. We provide a general meta-algorithm that leverages an online regression oracle to estimate the unknown safety constraint, and converts the predictions of an online learning oracle to predictions that adhere to the unknown safety constraint. On the theoretical side, our algorithm's regret can be bounded by the regret of the online regression and online learning oracles, the eluder dimension of the model class containing the unknown safety constraint, and a novel complexity measure that captures the difficulty of safe learning. We complement our result with an asymptotic lower bound that shows that the aforementioned complexity measure is necessary. When the constraints are linear, we instantiate our result to provide a concrete algorithm with $\sqrt{T}$ regret using a scaling transformation that balances optimistic exploration with pessimistic constraint satisfaction. 
\end{abstract}


\section{Introduction}
Online learning is a key tool for many sequential decision making paradigms. From a practical view point, it is often the case that either due to safety concerns \cite{dobbe2020}, to guarantee fairness or privacy properties \cite{ZafarVGG19}, \cite{LevySAKKMS21}, or in many cases, simply due to physical restrictions in the real world \cite{AtawnihPD16}, the agent or learner often must pick actions that are not only effective but also strictly adhering to some constraints on every round. Often, the safety constraint is determined by parameters of the environment that are unknown to the learner. For example individual fairness constraints may be defined by an unknown similarity metric \cite{GillenJKR18}, or in robotics applications, safety may hinge on uncertainties such as an unknown payload weight \cite{BrunkeGHYZPS22}. Thus in such situations, the learner must learn the unknown parameters that characterize the safety constraint. 

In this work, we study the general problem of online learning with unknown constraints, where the learner only observes noisy feedback of the safety constraints. We consider arbitrary decision spaces and loss functions. Our goal is to design algorithms that can simultaneously minimize regret while strictly adhering to the safety constraint at all time steps. Naturally, regret is measured w.r.t. the best decision in hindsight that also satisfies the constraint on every round. The learner only has access to an initial safe-set of actions/decisions to begin, and must gain more information about the safety constraint.

To solve this problem, we assume access to a general online learning oracle that has low regret (without explicit regard to safety) and a general online regression oracle that provides us with increasingly accurate estimations of the unknown constraint function. The key technical insight in this work is in exploring what complexity or geometry allows us to remain within guaranteed safe-sets while expanding the safe-sets and simultaneously ensuring regret is small. We introduce a complexity measure that precisely captures this inherent per-step tension between regret minimization and information acquisition with respect to the safety constraint (with the key challenge of remaining within the safe set).  We complement our results with a lower bound that shows that asymptotically, whenever this complexity measure is large, regret is also large. Our results yield an analysis that non-constructively shows the existence of algorithms for the general setting with arbitrary decision sets, loss functions, and classes of safety constraints. Furthermore, we instantiate these results explicitly for various settings, and give explicit algorithms for unknown linear constraints and online linear optimization. To the best of our knowledge the best algorithm in this setting has a regret bound of $\bigO(T^{2/3})$ while our algorithm for this specific setting attains a $\bigO(\sqrt{T})$ bound. We note the key contributions of our paper below. 
%
%
%
%

\paragraph{Key Contributions}
\begin{itemize}
    \item We provide a new safe learning algorithm under an unknown constraint (Algorithm \ref{alg:su:gen}) that utilizes a online regression oracle w.r.t. $\F$, where $\F$ is the model class to which the unknown safety constraint belongs, and an online learning oracle that guarantees good performance albeit completely agnostically of the safety constraint. Notably, our algorithm is able to handle adversarial contexts, arbitrary action set $\Action$ and model class $\F$, operates under general modelling assumptions, and enjoys the following regret bound:
    \begin{align*}
        \regret_T \leq \inf_{\kappa} \left\{ \sum_{t=1}^T V_t(\kappa) + \kappa \inf_{\alpha} \left\{  \alpha T + \frac{\Regsq (T,\delta, \F) \edim(\F, \alpha)}{\alpha} \right\} \right\} + \regretol(T,\delta)
    \end{align*}
    where $\Regsq (T,\delta, \F)$ denotes the regret bound guaranteed by the online regression oracle on $\F$, $\edim(\F, \alpha)$ denotes the eluder dimension of $\F$, and $\regretol(T,\delta)$ denotes the regret bound guaranteed by the online learning algorithm. $V_t(\kappa)$ is a complexity measure we introduce that captures the trade off between loss minimization and information gain w.r.t. unknown constraint.
    \item Via a lower bound, we show that asymptotically, if $\lim_{T \to \infty} 1/T \sum_{t=1}^T V_t(\kappa)$ is large, no safe algorithm is able to obtain diminishing regret.
    \item For linear \& generalized linear settings, we instantiate our result to a give a simple algorithm with $\bigO(\sqrt{T})$ regret.
    \item We extend our main algorithm to be able to handle multiple constraints
 \end{itemize}

\section{Related Works}
\textbf{Online Convex Optimization and Long Term Constraints} \\
\cite{mahdavi_trading_2012} initiated the problem of online convex optimization with long term constraints, a variant of online convex optimization where the learner is given a set of functional constraints $\{ \f_i(\cdot) \leq 0 \}_{i=1}^m$ and is required to ensure that the sum of constraint violations $\sum_{t=1}^T \sum_{i=1}^m \f_{i}(x_t)$ is bounded rather than ensuring that the constraints must be satisfied at every time step. For the case of known and fixed constraints, \cite{mahdavi_trading_2012} obtain $\bigO(T^{1/2})$ regret and $\bigO(T^{3/4})$ constraint violation. This was recently improved in \cite{yu_low_2020} to be $\bigO(T^{1/2})$ regret and $\bigO(1)$ constraint violation. Furthermore, \cite{neely_online_2017} study a variant with time varying constraints $\{\f_{i,t}(\cdot) \leq 0\}_{i=1}^m$ and and achieve $\bigO(T^{1/2})$ regret and long term constraint violation. \cite{sun_safety-aware_nodate}, \cite{jenatton_adaptive_2015}, and \cite{yi_distributed_2019} study variations of this problem. 



\noindent \textbf{Bandits with Unknown Linear Constraint} \\
The line of works that most resembles our own is that studying the safe bandits with unknown linear constraints. Initiated by \cite{moradipari_safe_2020}, this line of works studies a linear bandit setting, where a linear constraint is imposed on every action $\action_t$ of the form of $\langle \f, \action_t \rangle - b \leq 0$ with unknown $\f$ and known $b$. Similar settings involving linear bandit problems with uncertain and per-round constraints have been studied by \cite{amani_decentralized_2020}, \cite{pacchiano_stochastic_nodate}, \cite{hutchinson_exploiting_2023}. \cite{pacchiano_stochastic_nodate} study the version where the constraint must be satisfied in expectation, and \cite{pacchiano_contextual_2024} and \cite{hutchinson_exploiting_2023} improves this to a high probability. 

\noindent \textbf{Safe Convex Optimization with Unknown Linear Constraint(s)} \\
Safe convex optimization with unknown linear safety constraints was studied in \cite{usmanova_safe_2019}. \cite{fereydounian_safe_2020} seeks to optimize a fixed convex function given unknown linear constraints, and focuses on sample complexity. Closest to our work is that of \cite{chaudhary_safe_2022}, where the authors study time varying cost functions and achieve $\bigO(T^{2/3})$ regret. 

\noindent \textbf{Per Timestep Tradeoff Between Loss Minimization and Constraint Information Gain} \\
The SO-PGD algorithm due to \cite{chaudhary_safe_2022} adopts an explore-first then exploit strategy which results in a $O(T^{2/3})$ regret bound, whereas the ROFUL algorithm due to \cite{hutchinson_exploiting_2023} strikes a better balance between regret minimization and conservative exploration of the constraint set. The Decision Estimation Coefficient (DEC) due to \cite{foster_statistical_2023}, \cite{foster_complexity_2022} explicitly strikes a balance loss minimization and the information gained due to observation. Our proposed algorithm seeks to similarly balance loss minimization and exploration of the constraint set, in a manner reminiscent of DEC. 

\section{Setup and Preliminary}\label{sec:setup}
We consider the problem of online learning with unknown constraints imposed on actions the learner is allowed to play from. The learning problem proceeds for $T$ rounds as follows.
\begin{itemize}
\item[] For $t = 1, \ldots ,T$:
\begin{itemize}
\item {\em Adversary picks a context $\cont_t \in \Context$.
\item  Learner follows by picking a possibly randomized action $\action_t \in \Action$.
\item Adversary reveals outcome $\out_t \in \Out$ that specifies the loss $\ls(\action_t,\cont_t,\out_t)$. 
\item The learner request for constraint feedback $\zf_t  \sim \signal(f^\star(\action_t,\cont_t))$.}
\end{itemize}
\end{itemize}

\noindent We assume a \textit{full information} feeddback setting with respect to the losses, where the learner observes the choice of the  adversary $\out_t \in \Out$. We assume that losses are bounded, $\ell: \Action \times \Context \times \Out \to [0,1]$.
Additionally, the learner does not know the constraint function $\f^\star$ but only that $\f^\star \in \F$ for some $\F \subseteq \Action \times \Context \to [-1,1]$.  The goal of the learner is to minimize:
$$
\regret_{T} := \sum_{t=1}^T \ls(\action_t,\cont_t,\out_t) - \min_{\substack{\action \in \Action: \forall t,\\  \f^\star(\action,\cont_t) \le 0}} \sum_{t=1}^T \ls(\action,\cont_t,\out_t)~,
$$
the regret with respect to the optimal action $\action$ in hindsight that satisfies the constraint $\f^\star(\action,\cont_t) \le 0$ on every round $t$. The learner in turn is also only allowed to take actions $\action_t$ s.t. $\f^\star(\action_t,\cont_t) \le 0$.


Since we are interested in making no constraint violations while taking our actions, learning is impossible unless we are at least given an initial set of actions $\Action_{0} \subseteq \Action$ that is guaranteed to be safe under any context and any $f \in \F$. We will assume that we are given such a $\Action_0$. 
\subsection{Additional Notation}
We use the notation $\Pi_S(x)$ to denote the projection of a vector $x \in \mathbb R^d$ onto a set $S \subseteq \mathbb R^d$. For a positive definite matrix $M \in \mathbb R^{d \times d} $ and vector $x \in \mathbb R^d  $ we denote the norm induced by $M$ as $\|x\|_M := \sqrt{x^\intercal M x}$. We denote the convex hull of a set $S$ as $\conv(S)$. Let $\{x_s\}_{s=1}^t := x_1, \ldots, x_t$ be shorthand for a sequence. For a function class $\F$, we denote $\Delta_{\F}(\x) := \sup_{\f, \f' \in \F} f(\x) - f'(\x)$. For a set $\G$, we denote $\Simplex(\G)$ as the set of distributions over $\G$. We adopt a non-asymptotic big-oh notation: for functions $f, g : \mathcal X \to \reals_+$, $f = \bigO(g)$ if there exists some constant $C>0$ such that $f(x) \leq C g(x)$ for all $x \in \mathcal X$. We write $f = \smallO(g)$ if for every constant $c$, there exists a $x_0$ such that $f(x) \leq cg(x)$ for all $x \geq x_0$.  

\subsection{Online Regression Oracles and Signal Functions}
Similar to prior works concerned with estimation of function classes \cite{foster_practical_2018}, \cite{foster_beyond_2020}, \cite{foster_instance-dependent_2020}, \cite{sekhari_contextual_2023}, \cite{sekhari_selective_2023}, we assume our algorithms have access to an online regression oracle, $\oraclesq$. However, unlike these prior works that assume that the provided online regression oracle enjoys a sublinear regret bound, we require that our oracle satisfy a slightly weaker condition that allows for algorithms geared towards realizability. 
\begin{assumption}[Online Regression Oracle]\label{as:su:oraclesq}
The algorithm $\oraclesq$ guarantees that for any (possibly adversarially chosen) sequence $\{\action_t,\cont_t\}_{t=1}^T$, for any $\delta \in (0,1)$, with probability at least $1-\delta$, generates predictions $\{\hzf_t\}_{t=1}^T$ satisfying:
$$
\sum_{t=1}^T(\hzf_t - \f^*(\action_t,\cont_t))^2 \le \Regsq (T,\delta, \F)
$$
\end{assumption}

Assumption \ref{as:su:oraclesq} is closely linked with the model $p_{\mathrm{signal}}$ that produces feedback about constraint value, and any regret-minimizing oracle for strongly convex losses can be converted into an oracle that satisfies the assumption with high probability. We formalize this in Lemma \ref{lem:ap:regrettosquareloss} in the appendix. For instance, if $\zf_t \sim \signal(\f^\star(\action_t,\cont_t))$ is given by $\zf_t = \f^\star(\action_t,\cont_t) + \eg_t$ where $\eg_t$ is any  sub-gaussian distributed random variable, then any online square loss regression algorithm on class $\F$ that uses $\zf_t$ as corresponding outcomes will satisfy Assumption \ref{as:su:oraclesq}. Similarly, if $\zf_t \in \{0,1\}$ is drawn as $ P(\zf_t=1| \f^\star(\action_t,\cont_t)) \propto \exp(\f^\star(\action_t,\cont_t)) = \signal(\f^\star(\action_t,\cont_t)) $ (ie. the Boltzman distribution), then one can show that Assumption \ref{as:su:oraclesq} is satisfied by running any online logistic regression algorithm over class $\F$ with $\zf_t$ as labels.

\cite{rakhlin_online_2014} characterized the minimax rates for online square
loss regression in terms of the offset sequential Rademacher complexity for arbitrary $\F$.  This for example, leads to regret bounds of the form, $\mathrm{Reg}_{\mathrm{OR}}(T,\F) = \bigO(\log|\F|)$ for finite function classes $\F$, and $\mathrm{Reg}_{\mathrm{OR}}(T,\F) =
\bigO(d \log(T))$ when $\F$ is a $d$-dimensional linear class. More examples can be found in \cite{rakhlin_online_2014} (Section 4). 
\subsection{Online Learning Oracles}
Next, we consider the following online optimization problem that is important for solving the learning problem with constraints introduced in this section. We consider an online learning problem where on every round $t$, adversary first announces a set $\Action_t \subseteq \Action$. Next learner picks a distribution $p_t \in \Simplex(\Action_t)$ and draws $\action_t \sim p_t$. Adversary then produces a loss function $\out_t \in \Out$. The learner suffers loss $\ell(\action_t,\out_t)$. The goal of the learner is to minimize regret w.r.t. the best action  $\action \in \cap_{t=1}^T \Action_t$. We assume access to an online learning oracle, $\oracleol$ that satisfies the following assumption: 
\begin{assumption}[Online Learning Oracle]\label{as:su:oracleol}
For any sequence of adversarially chosen sets $\{\Action_t\}_{t=1}^T$ and any $\delta \in (0,1)$ with probability at least $1 - \delta$, the algorithm $\oracleol$ produces a sequence of distributions $\{ p_t \}_{t=1}^T$ satisfying $p_t \in \Simplex(\Action_t)$ for all $t \in [T]$ with expected regret bounded as:
    $$
    \sum_{t=1}^T \Ex{\action_t \sim p_t}{\ell(\action_t,\cont_t,\out_t)} - \min_{\action \in \cap_{t=1}^T \Action_t} \sum_{t=1}^T \ell(\action,\cont_t,\out_t) \le \regretol(T,\delta)
    $$
\end{assumption}
The reader might wonder how for arbitrary choice of sets $\Action_t$ chosen by the adversary, such a regret minimizing algorithm with $o(T)$ regret is even possible. To this end, in the following proposition we show that as long as losses are bounded by $1$ (or more generally any $B$), one can use online symmetrization arguments along with a minimax analysis of the above game and conclude a bound on regret in terms of sequential Rademacher complexity of the loss class. Specifically, we denote the sequential Rademacher complexity of the loss class by
$$
\mathrm{Rad}^{\mathrm{seq}}_{\ell \circ \X}(T) := \sup_{\mathbf{\out}, \mathbf{\cont}} \mathbb{E}_\epsilon\left[ \sup_{\action \in \Action} \sum_{t=1}^T \epsilon_t \ell\left(\action,\mathbf{\cont}_t(\epsilon_{1:t-1}), \mathbf{\out}_t(\epsilon_{1:t-1})\right)\right]
$$
where in the above supremum over $\mathbf{\out}$ and $\mathbf{\cont}$ are taken over all mapping of the form $\mathbf{\out} : \bigcup_{t=0}^{T-1} \{\pm1\}^{t}\mapsto \Out$ and $\mathbf{\cont} : \bigcup_{t=0}^{T-1} \{\pm1\}^{t}\mapsto \Context$ respectively.
\begin{proposition}\label{prop:su:nonconstructive_existence}
There exists an algorithm satisfying Assumption \ref{as:su:oracleol} with 
$$
\regretol(T,\delta) \le 2\ \mathrm{Rad}^{\mathrm{seq}}_{\ell \circ \X}(T) 
$$ 
\end{proposition}
We prove the above proposition using similar of symmetrization arguments as made in \cite{RST10}. Various examples of bounds on $\mathrm{Rad}^{\mathrm{seq}}_{\ell \circ \X}(T)$ and its various properties can be found in \cite{RST10}. Notably, this result is non-constructive and only guarantees the existence of such regret minimizing oracles. In the proceeding section we provide a concrete gradient-descent based algorithm in the online linear optimization setting.  
\subsection{Eluder Dimension}
Before delving into our main results, we recall the following definition of $\epsilon$-dependence and eluder dimension \cite{russo_eluder_nodate}, \cite{foster_instance-dependent_2020}, \cite{foster_statistical_2023}. 

\begin{definition}$~$
\begin{itemize}
	\item An action, context pair $(\action,\cont) \in \Action \times \Context$ is $\epsilon$-dependent on $\{\action_i, \cont_i\}_{i=1}^t \subseteq \Action \times \Context$ w.r.t. $\F$ if every $\f,\f' \in \F$ satisfying $\sqrt{\sum_{i=1}^t (\f(\action_i, \cont_i) - f'(\action_i, \cont_i))^2} \leq \epsilon$ also satisfies $\f(\action, \cont) - \f'(\action, \cont) \leq \epsilon$. $(\action, \cont)$ is $\epsilon$-independent w.r.t. $\mathcal F$ if $\action$ is not $\epsilon$-dependent on $\{(\action_i,\cont_i)\}_{i=1}^t$. 
	\item The eluder dimension $\Eluder(\F,\epsilon)$ is the length of the longest sequence of pairs in $\Action \times \Context$ such that for some $\epsilon' > \epsilon$, each pair is $\epsilon'$-independent of its predecessors. 
\end{itemize}
\end{definition}
The eluder dimension is bounded for a variety of function classes. For example, when $\F$ is finite, $\Eluder(\F,\epsilon) \leq |\F| -1$, and when $\F$ the class of linear functions, $\Eluder(\F,\epsilon) \leq \bigO(d \log (1/\epsilon))$. The eluder dimension of function class $\F$ will be a component of our regret bounds.

\section{Main Results}\label{sec:main_results}
In this section we provide the main results of our paper. Specifically, in the first subsection of the paper, we propose a generic algorithm with corresponding upper bound for the problem of online learning with unknown constraints. We also provide a lower bound that shows that at least asymptotically our upper bound captures the key complexity of the problem. We also supplement our results with one of the so called long term constraints where we show that in general when one only wants constraint values to be small on average rathe than exact constraint satisfaction, then many of the complexities of the problem disappear. 

\subsection{Algorithm and Upper Bound}
Now given the setting and the oracle assumptions, we are ready to sketch the high level idea of our main algorithm, presented in Algorithm \ref{alg:su:gen}. First notice that we are assuming access to an online regression oracle $\oraclesq$. On every round, given action played $\action_t$ and context $\cont_t$ on the round, we can make online prediction on noisy observation $\zf_t$ that has low regret. Since $\zf_t$ is an unbiased estimate of $f^\star(\action_t,\cont_t)$, the regret bound ensures that sum of squares deviation $\sum_{t} (\hzf_t - f^\star(\action_t,\cont_t))^2$ is small. Using this, we build a version space $\F_t \subset \F$ such that with high probability $f^\star \in \F_t$ (using the regret bound guarantee of squared loss online regression oracle). Next, using this set, we build on every round $t$, an optimistic set of actions $\Optimistic_t$ that is a super set of all actions that satisfy constraints (might contain some that don't satisfy constraint as well) for the round. 
\begin{proposition}\label{prop:mr:optimism}
Let $\{\Optimistic_t\}_{t=1}^T$ be the sequence of optimistic sets generated by Algorithm \ref{alg:su:gen}. For any $t \in [T]$ and any $\delta \in (0,1)$ with probability at least $(1-\delta) $ we have,
$
\{\action \in \Action: f^\star(\action,\cont_t) \le 0\} \subseteq \Optimistic_t$
\end{proposition}
We also construct a pessimistic set of actions $\Pessimistic_t$ of actions that for that rounds are guaranteed to satisfy the constraint for that round (some actions that do satisfy constraint might be left out). 
\begin{proposition}\label{prop:mr:pessimism}
Let $\{\Pessimistic_t\}_{t=1}^T$ be the sequence of optimistic sets generated by Algorithm \ref{alg:su:gen}. For any $t \in [T]$ and $\action \in \Pessimistic_t$ and for any $\delta \in (0,1)$ with probability at least $(1-\delta) $ we have constraint satisfaction, i.e. $f^\star(\action,\cont_t) \le 0$.
\end{proposition}
We use the optimistic set $\Optimistic_t$ along with given context $\cont_t$ as inputs to the online learning oracle $\oracleol$ and receive a recommended distribution $\tilde p_t \in \Simplex(\Optimistic_t$). Because $\Optimistic_t$ contains all constraint-satisfying actions with high probability, combined with the regret guarantee of $\oracleol$ in Assumption \ref{as:su:oracleol}, this guarantees that actions drawn from $\tilde p_t$ will have good performance. 

However, in order to ensure constraint satisfaction, we need to play actions from the pessimistic set $\Pessimistic_t$. To this end, we introduce a mapping $\Map$ that takes in a distribution over $\Optimistic_t$ along with the pessimistic set $\Pessimistic_t$, function class $\F_t$ and context $\cont_t$ and returns a distribution over $\Pessimistic_t$ \footnote{Wherever it is clear what the arguments to the mapping are we may drop them e.g. in a non-contextual setting, the context argument is dropped}. The exact properties we desire of mapping $\Map$ will be discussed later. We draw action $\action_t$ from the distribution defined by $\Map$ for the round, receive $\out_t$ and noisy feedback $\zf_t$ which we in turn pass as input to the online learning oracle $\oracleol$ and regression oracle $\oraclesq$ respectively. 
\begin{algorithm}
\caption{General Constrained Online Learning}
\label{alg:su:gen}
\begin{algorithmic}[1]
\STATE Input:  $\oracleol$, $\oraclesq$, Initial safe set $\Action_0$, $\delta \in (0,1)$
\STATE $\F_0 = \{\f \in \F: \forall \action \in \Action_0, \forall \cont \in \Context, \f(\action,\cont) \le 0\}$
\FOR{$t = 1,\ldots,T$}
\STATE {\bf Receive context $\cont_t$}
\STATE $\F_t = \{ \f \in \F_0 : \sum_{s=1}^{t-1}(\f(\x_s, \cont_s) - \hzf_s )^2 \leq \Regsq (T,\delta, \F_0) \}$
\STATE $\Optimistic_t = \left\{ \action \in \Action : \min_{f \in \F_t} f(\action,\cont_t)  \leq 0 \right\}$ ~~,~~ $\Pessimistic_t = \left\{ \action \in \Action : \max_{f \in \F_t} f(\action,\cont_t) \leq 0 \right\}$ 
\STATE $\tilde p_t = {\oracleol}_t(\cont_t,\Optimistic_t)$
\STATE $p_t =  \Map(\tilde p_t; \Pessimistic_t, \F_t,\cont_t)$
\STATE 	Draw  $\action_t \sim p_t$
\STATE {\bf Ask for noisy feedback $\zf_t$ }
\STATE Update $\hzf_t = {\oraclesq}_t(\cont_t,\action_t)$
\STATE {\bf Play $\action_t$ and receive $\out_t$}
\ENDFOR
\end{algorithmic}
\end{algorithm}

\begin{theorem} \label{thm:main:maintheorem}
For any $\delta \in (0,1)$ with probability at least $1-3\delta$, Algorithm \ref{alg:su:gen} produces a sequence of actions $\{\action_t \}_{t=1}^T$ that are safe, and enjoys the following bound on regret: 
\begin{align*}
\regret_{T} &\le \inf_{\kappa >0}\left\{\sum_{t=1}^T V_\kappa(\tilde p_t; \Pessimistic_t,\F_t, \cont_t) +  \kappa \inf_{\alpha}\left\{\alpha T + \frac{20 \Regsq (T,\delta, \F_0) \mathcal{E}(\F_0,\alpha)}{\alpha} \right\} \right\} \\
&+ \regretol(T,\delta) + \sqrt{2 T \log (\delta^{-1})}
\end{align*}
where, 
\begin{align*}
V_\kappa(\tilde p_t; \Pessimistic_t,\F_t,\cont_t)  &= \sup_{\out \in \Out}\left\{ \Ex{\action_t \sim  \Map(\tilde p_t; \Pessimistic_t, \F_t,\cont_t)}{\ell(\action_t,\cont_t,\out)} - \Ex{\tilde a_t \sim \tilde p_t}{\ell(\tilde \action_t,\cont_t,\out)}\right\} -  \kappa \Ex{\action_t \sim  \Map(\tilde p_t; \Pessimistic_t, \F_t,\cont_t)}{\Delta_{\F_{t}}(\action_t,\cont_t)}
\end{align*}
Further, if we use~~ 
$
\kappa^* = \max_{t \in [T]} \sup_{\cont \in \Context, \tilde p \in \Delta(\Pessimistic_t), \out \in \Out} \frac{\Ex{\action \sim  \Map(\tilde p; \Pessimistic_t, \F_t,\cont)}{\ell(\action,\cont , \out)} - \Ex{\tilde \action \sim \tilde p}{\ell(\tilde \action,\cont, \out)}}{\Ex{\action \sim  \Map(\tilde p; \Pessimistic_t, \F_t,\cont)}{\Delta_{\F_t}(\action,\cont)}}
$, then in the above, $V_{\kappa^*}(\tilde p_t; \Pessimistic_t,\F_t,\cont_t) \le 0$ and so we can conclude that:
\begin{align*}
\regret_{T} & \le  \kappa^* \inf_{\alpha}\left\{\alpha T + \frac{20 \Regsq (T,\delta, \F_0) \Eluder(\F_0,\alpha)}{\alpha} \right\} +  \regretol(T,\delta)  + \sqrt{T \log (\delta^{-1})}
\end{align*}
\end{theorem}

\begin{remark}
	Because $\F_0 \subseteq \F$, we choose to have the regression oracle $\oraclesq$ and the eluder dimension depend on $\F_0$ - consequently we have terms $\Regsq (T,\delta, \F_0)$, and $\mathcal{E}(\F_0,\alpha)$ in our bound. However, if they are smaller, we may use $\Regsq (T,\delta, \F)$, and $\mathcal{E}(\F,\alpha)$ instead. 
\end{remark}

\paragraph{Optimal Mapping and Adapting to $\kappa$} In the next section we give concrete examples of the mapping $\Map$ which yields bounded $\kappa^*$ for a few settings. But this begs the question - is it possible to write down a form for the mapping for the general case? For a given $\kappa>0$, the worst case optimal mapping that minimizes $V_\kappa$ can be written down as a solution to a saddle point optimization problem:
\begin{align*}
	\Map^\kappa(\tilde{p}; \Pessimistic,\G,\cont) = \argmin_{p \in \Simplex(\Pessimistic)} \min_{g, g' \in \G}\sup_{\out \in \Out} \left\{ \Ex{\action \sim p_t}{\ell(\action,\cont,\out)} - \Ex{\tilde \action \sim \tilde p_t}{\ell(\tilde \action,\cont,\out)}\} - \kappa \Ex{\action \sim p_t}{|g(\action) - g'(\action)|} \right\}
\end{align*}
where we use the definition of $\Delta_\G(\cdot)$. Notice that the above optimization depends on $\kappa$, and the optimal $\kappa$ can only be computed in hindsight. A natural question is then: how do we pick $\kappa$? In order to adapt to $\kappa$ as we learn, we first notice that since $\inf_{\alpha}\left\{\alpha T + \frac{20\Regsq (T,\delta, \F_0) \Eluder(\F_0,\alpha)}{\alpha} \right\} \geq \sqrt{T}$, $\kappa$ ranges from $(0,  \sqrt{T} ] $ in order for the regret bounds to be $\smallO(T)$. Therefore, we could choose from a range of candidate $\kappa$'s, $\kappa \in \{2^{i}\}_{i=0}^{\log(\sqrt{T})}$. A standard trick we could use to obtain regret bounds as good as the best $\kappa$ in this set (which is at most twice the optimal amongst any $\kappa >0$) is to simply use a non-stochastic multiarmed bandit algorithm with each arm representing each of the $\log(\sqrt{T})$ values of $\kappa$ under consideration. The loss of the a particular choice of $\kappa$ is then simply the loss of the action obtained using the mapping $\Map^\kappa$. Using the EXP3 algorithm \cite{EXP3} for example, an algorithm that adaptively picks $\kappa$ obtains a regret bound up to a constant factor of the regret obtained using the best $\kappa$, and only suffers an additive factor of $\bigO(\sqrt{T\log(T)})$. This algorithm remains safe since we only play from $\Pessimistic_t$.

\paragraph{Long Term Constraint Versus No Violations with High Probability}
Suppose we are only interested in ensuring that the sum of constraint violations $\sum_{t=1}^T \f^\star(\action_t,\cont_t)$ is $\smallO(T)$, as is the goal in the line of works studying online learning with long term constraints \cite{mahdavi_trading_2012}, \cite{yu_low_2020}, \cite{sun_safety-aware_nodate}. Then, we can bound the number sum of constraint violations by the eluder dimension. Furthermore, this can be done by leveraging the online learning oracle and online regression oracle, without requiring the use of the mapping $\Map$ present in the main algorithm - the idea will be to simply play the output of the online learning algorithm given sets $\{ \Optimistic_t \}_{t=1}^T$. Algorithm \ref{alg:ap:longterm} defined in the appendix has the following guarantee: 
\begin{lemma}\label{lem:mr:longterm}
    For any $\delta \in (0,1)$ with probability at least $1-2\delta$, Algorithm \ref{alg:ap:longterm} produces a sequence of actions $\{ \action_t \}_{t=1}^T$ that satisfies: 
    \begin{align*}
        \regret_T \leq \regretol(T,\delta) ~\textrm{ and }~\sum_{t=1}^T \f^*(\action_t, \cont_t) \leq  \inf_{\alpha}\left\{\alpha T + \frac{20 \Regsq (T,\delta, \F_0) \edim(\F_0,\alpha)}{\alpha} \right\}
    \end{align*}
\end{lemma}

This motivates the question: is assuming access to an online learning oracle and online regression oracle and that the eluder dimension of $\F$ is small enough for us to create algorithms that make no constraint violations with high probability? 

Unfortunately the answer is a no. We need more assumptions on the initial safe set - which is what the mapping is harnessing. Specifically consider the case where $\ell(\action,\cont, y) = \out^\top \action$, the constraint set is $\F = \{ (\action,\cont) \mapsto f^\top \action : \|\f\|_2 \le 1\}$ and say the initial safe action set $\Action_0 = \{0\}$. In this case, the eluder dimension $\Eluder(\F,\epsilon) = d \log(1/\epsilon)$, and both the online learning oracle and online regression oracle are readily available and satisfy Assumptions \ref{as:su:oraclesq} and \ref{as:su:oracleol} (e.g. use gradient descent and online linear regression algorithm). However, since $\Action_0 = \{0\}$ the initial pessimistic set is $\Pessimistic_1 = \{0\}$. However since we are forced to play in this set, we don't gain any information about $\f^\star$ and hence in the subsequent rounds $\F_t = \F$ and $\Pessimistic_t = \Pessimistic_1$. Thus we cannot hope to play anything other than the single safe choice $\action_0 = 0$ which prevents us from achieving low regret.   

\subsection{Lower Bound}
The assumption of having an online learning oracle that guarantees low regret (while playing from the optimistic sets) is a natural one since it guarantees the existence of a regret-minimizing algorithm. The existence of an online regression oracle and assumption that eluder dimension is well behaved are also assumptions that are typically expected. While perhaps the eluder dimension may be substituted by a milder star number \cite{sekhari_selective_2023}, we would nonetheless expect to see some measure of complexity of $\F$. The assumption that the reader would be unfamiliar with is perhaps the one where we assume existence of a mapping $\Map$ from a distribution over optimistic sets to a distribution over pessimistic sets that ensures that the sum $\sum_{t=1}^T  V_\kappa(\tilde p_t; \Pessimistic_t,\F_t, \cont_t) +   2 \kappa \inf_{\alpha}\left\{\alpha T + \frac{20 \Regsq (T,\delta, \F_0) \mathcal{E}(\F_0,\alpha)}{\alpha} \right\}  $ is small in our main bound for some reasonable $\kappa$. In this section we show that at least asymptotically, the existence of such mapping is necessary to even guarantee that regret can be diminishing. To make things simpler, for this lower bound we will ignore context and assume that context set is the null set. We will also fix losses to be the same on all rounds and even assume that the learner knows the loss value of the optimal safe action. Throughout this section we assume that $\zf_t  \sim \signal(f^\star(\action_t,\cont_t))$  is of the form $\zf_t = \f^\star(\action_t,\cont_t) + \eg_t$  where $\eg_t$'s are  standard normal noise variables so that we can use square loss regression for the online regression oracle.  

We will show that if for some $\kappa$, and any mapping $\Map$,   $\lim_{T\rightarrow \infty} \frac{1}{T} \sum_{t=1}^T  V_\kappa(\tilde p_t; \Pessimistic_t,\F_t) \ge c^*> 0$, then,  there exists $\Pessimistic^* \supset \Action_0$ such that $\Pessimistic^*$ is safe, (constraints are always satisfied), for any action $\action \in \Pessimistic^*$ with high probability we can estimate $\f^\star(a)$ to any arbitrary accuracy, yet all actions in $\Pessimistic^*$ are sub-optimal in terms of loss by at least $c^*$ when compared to best safe action. Further, this set $\Pessimistic^*$ is non-expandable - meaning that based on knowing the values of $\f^\star(a)$ for every $a \in \Pessimistic^*$, and knowing $\F$, we cannot find more actions that are guaranteed to be safe. 

Once we can show the existence of such a set $P^*$, our work is easy. We can simply announce to any learning algorithm the initial safe set as $\Action_0 = \Pessimistic^*$ which only gives the learning algorithm more knowledge. Now because the set is non-expandable, and since the learning algorithm is only allowed to play actions guaranteed to be safe, we conclude any algorithm is doomed to always play actions within $\Pessimistic^*$ - which is known to be $c^*$ sub-optimal. Unfortunately, since we would like to drive estimation error of $f^*$ on elements of $P^*$ to zero - we are only able to do this asymptotically. We capture this argument in the following lemma. 

\begin{lemma}\label{lem:lb2}
Assume that we have a fixed loss function $\ell:\Action \mapsto \reals$ such that for any $\action \in \Action$ satisfying $\f^\star(\action) > 0$, $\ell(\action) = \min_{\action^* \in \Action: \f^\star(\action^*) \le 0 } \ell(\action^*)$. Furthermore, assume that the eluder dimension of $\F$ at any scale $\epsilon > 0$, (with input space $\Action$) is bounded. If for some $c^* > 0$, $\kappa \ge 0$, and any $\Map$, any regret minimizing oracles $\oracleol$ and $\oraclesq$ (assuming regret in both cases is $\smallO(T)$)
$
\lim_{T\rightarrow \infty} \frac{1}{T} \sum_{t=1}^T  V_\kappa(\tilde p_t; \Pessimistic_t,\F_t) \ge c^*
$
then, there exists a set $\Pessimistic^* \supseteq \Action_0$ with the following properties,
\begin{enumerate}
\item Set $\Pessimistic^*$ satisfies constraints, i.e. $\forall \action \in \Pessimistic^*$, $\f^\star(a) \le 0$
\item Define $\F^* = \{ f : \forall \action \in  \Pessimistic^*,~ f(a) = f^\star(a)\}$. For every action $a \in \Action \setminus  \Pessimistic^*$, $\exists f \in \F^*$ such that $f(a) > 0$.  That is,  $\Pessimistic^*$ cannot be expanded to a larger set guaranteed to satisfy constraint.
\item $\Pessimistic^*$ is such that  $\inf_{\action \in  \Pessimistic^*} \ell(a) - \inf_{\action \in \Action: f^\star(\action) \le 0} \ell(a) \ge c^\star$
\end{enumerate}
\end{lemma}

\begin{proposition}\label{prop:lb1}
If there exists a set $\Pessimistic^*$ that has the following properties,
\begin{enumerate}
\item Set $\Pessimistic^*$ satisfies constraints, i.e. $\forall \action \in \Pessimistic^*$, $\f^\star(a) \le 0$
\item Define $\F^* = \{ f : \forall \action \in  \Pessimistic^*,  f(a) = f^\star(a)\}$. For every action $a \in \Action \setminus  \Pessimistic^*$, $\exists f \in \F^*$ such that $f(a) > 0$.  That is,  $\Pessimistic^*$ cannot be expanded to a larger set guaranteed to satisfy constraint.
\item $\Pessimistic^*$ is such that  $\inf_{\action \in  \Pessimistic^*} \ell(a) - \inf_{\action \in \Action: f^\star(\action) \le 0} \ell(a) \ge c^\star$
\end{enumerate}
Then, safe learning is impossible, and any learning algorithm that is guaranteed to satisfy constraints on every round (with high probability) has a regret lower bounded by $\regret_T \ge T c^*$.
\end{proposition}

The following theorem states that even when eluder dimension is well behaved and we have access to online learning and online regression oracles with good bounds on regret, our assumption on the existence of a mapping $\Map$ that yields low regret is in fact necessary to ensure learning with constraint satisfaction on all rounds.

\begin{theorem}
Suppose we are given an initial action set $\Action_0$, $\Context = \{\}$, $\Out = \{\}$, a choice of constraint function $\f^\star \in \F$ (unknown to learner) and a fixed loss function $\ell:\Action \mapsto \reals$ such that for any $\action \in \Action$ satisfying $\f^\star(\action) > 0$, $\ell(\action) = \min_{\action^* \in \Action: \f^\star(\action^*) \le 0 } \ell(\action^*)$. Further assume that $\F$ has eluder dimension that is finite for any scale $\epsilon$. In this case, if for some $\kappa >0$, and any mapping $\Map$, if it is true that 
$$
\lim_{T\rightarrow \infty} \frac{1}{T} \sum_{t=1}^T  V_\kappa(\tilde p_t; \Pessimistic_t,\F_t) > 0,
$$
then, safe learning is impossible, i.e. no algorithm that is guaranteed to satisfy constraints on every round (with high probability) can ensure that $\regret_T = o(T)$.
\end{theorem}
\begin{proof}
Combining Lemma \ref{lem:lb2} and Proposition \ref{prop:lb1} trivially yields the statement of the theorem.
\end{proof}

\section{Examples}\label{sec:examples}
In this section, we restrict our attention to the case when we do not receive a context, i.e. $\Context = \emptyset$ and the loss functions are lipschitz. With a slight abuse of notation, we denote $\ell(\x_t, \cdot, \out_t) = \ell_t(\x_t)$. Notably, we will give examples of mappings $\Map$ that yield bounded $\kappa^*$ which allows us obtain concrete bounds from Theorem \ref{thm:main:maintheorem}.

\subsection{Finite Action Spaces}\label{subsection:finteactionspaces}
We first consider the setting of finite action spaces, where $\Action = [K]$, $\F_{\mathrm{FAS}} \subseteq \Action \to \reals$, and losses are functions $\ell_t : \Action \to [0,1]$. As shorthand, we define the vector $\ell_t \in \reals^K$ as the vector of loss values on each action and $\f \in \reals^K$ as the vector of constraint function values on each action. Suppose we make the following assumption that promises some separation between function values:
\begin{assumption}\label{as:ex:mab}
	All functions in $\F_0$ are separated on some action in $\Action_0$ by some $\Delta_0 > 0$: 
	\begin{align*}
		\min_{\f,\f' \in \F_0} \max_{\action \in \Action_0} \left\{ \f(\action) - \f(\action) \right\} \geq \Delta_0
	\end{align*}
\end{assumption}
This assumption is motivated by the fact that if for some timestep $t \in [T]$, for all $a \in \Pessimistic_t$ and $\f,\f' \in \F_t$, $\f(a) = \f'(a)$, then we have no hope of shrinking $\F_t$, and consequently expanding $P_t$. If the adversarial losses have values of $1$ for all $a \in \Pessimistic_t$, and values of $0$ for all $\action \notin \Pessimistic_t$, we'd suffer constant loss for all subsequent rounds. Therefore, some degree of separation is required - and we coarsely make this assumption for the first timestep. 

Now, at a given timestep $t \in [T]$ and $\tilde p_t$, $O_t$, $P_t$, $\F$ generated by Algorithm \ref{alg:su:gen}, we define a mapping $\Map(\tilde p_t, P_t, \F_t)$. Let $\action_\Delta := \argmax_{\action \in \Action} \Delta_{\F_t}(\action)$ be the width-maximizing action, $m_t := \sum_{i \in \Optimistic_t \setminus P_t} \tilde p_t[i]$ be the mass of $\tilde p_t$ outside of $\Pessimistic_t$, and let $\tilde p_t'$ be defined as:
\begin{align*}
	\tilde p_t' [\action] = \begin{cases}
		\tilde p_t [\action] + \frac{m_t}{|\Pessimistic_t|} & \action \in \Pessimistic_t \\
		0 & \action \notin \Pessimistic_t \\
	\end{cases} 
\end{align*}
Let $\gamma = \ind{|\F_t| > 1}$ and define the mapping 
\begin{align}\label{eq:ex:mabmapping}
	\Map_t(\tilde p_t, \Pessimistic_t, \F_{t}) := \gamma \mathbf{e}_{\action_\Delta} + (1-\gamma) \tilde p_t'
\end{align} 
This mapping is essentially an explore then exploit algorithm with respect to $\F_{\mathrm{MAB}}$ - it plays the maximum width action until we are sure what the true $f^*$. We show that such a mapping in Algorithm \ref{alg:su:gen} has bounded $\kappa^*$. 
\begin{lemma}\label{lem:ex:mabrelaxation}
Suppose $\F = \F_{\mathrm{FAS}}$ in Algorithm \ref{alg:su:gen} and suppose assumption \ref{as:ex:mab} holds. Suppose we use the mapping defined in equation \ref{eq:ex:mabmapping}. Then, $\kappa^* \leq \frac{1}{\Delta_0}$. 
\end{lemma}
\subsection{Linear Constraints}\label{subsection:linearconstraints}
\begin{assumption}\label{as:ex:lipschitz}
    The action set $\X$ is convex, compact, and bounded, $\max_{\x \in \X} \|\x\|_2 \leq D_\x$. The losses are lipschitz with constant $D_\ell$, $\forall t \in [T], \forall \x, \x' \in \X, | \ell_t(\x) - \ell_t(\x') | \leq D_\ell \| \x - \x ' \|$. The constants $D_\x, D_\ell$ are known to the learner. 
\end{assumption}

We consider the setting of linear constraints where:
\begin{align*}
    \F_{\mathrm{Linear}} = \{ (\x, \cont) \mapsto \langle \f , \x \rangle - b| \f \in \mathbb R^d \} 
\end{align*}
and the unknown constraint is $\langle \f^* , \x \rangle - b \leq  0$. Suppose that we are promised an initial safe set of $\Action_0 = \{ \action \in \Action: \|\action\| \leq b \}$. Then, any $\F_0 = \{ \f \in \reals^d : \|f \| \leq 1 \}$, since $\f$ with $\| \f \| > 1$ has $\langle f , b \frac{f}{||f||}\rangle -b > 0$, yet $b \frac{f}{||f||} \in \Action_0$ violating the promised initial safe set. We show that using a scaling-based mapping $\Map$ in Algorithm \ref{alg:su:gen} allows us to bound $\kappa^*$. 
\begin{lemma}\label{lem:ex:linearrelaxation}
Suppose $\F = \F_{\mathrm{Linear}}$ in Algorithm \ref{alg:su:gen} and suppose assumption \ref{as:ex:lipschitz} holds. Let $\gamma_t(\tilde \action_t):= \max\left\{ \gamma \in [0,1] : \gamma \tilde \action_t \in \Pessimistic_t \right\} $, and sample $\action_t \sim  \Map_t(\tilde p_t, \Pessimistic_t, \F_t)$ by drawing $\tilde \action_t \sim \tilde p_t$ then outputting $\gamma_t(\tilde \action_t)\tilde \action_t$. Then, $\kappa^* \leq \frac{D_\ell D_\x}{b}$.
\end{lemma}
We set $\oraclesq$ to be the Vovk-Azoury-Warmuth forecaster \cite{vovk_competitive_nodate}, \cite{azoury_relative_2001} which satisfies Assumption \ref{as:su:oraclesq} with $\Regsq(T, \delta, \F) \leq \mathcal{O} (d \log(\frac{T}{d\delta}))$. Furthermore, in the case of linear losses, $\ell_t(\action_t) = \langle \ell_t , \action_t \rangle$ we provide an online gradient descent based algorithm satisfying Assumption \ref{as:su:oracleol} in Algorithm \ref{alg:ap:linearconstructiveoracleol}, stated in the appendix. It is a randomized algorithm playing elements of the convex hull of $\Optimistic_t$. 
\begin{lemma}\label{lem:ex:linearconstructiveoracleol}
    When losses are linear, using Algorithm \ref{alg:ap:linearconstructiveoracleol} as $\oracleol$ satisfies Assumption \ref{as:su:oracleol} with:
    \begin{align*}
        \regretol(T,\delta)  \leq 4D_\f D_\x \sqrt{T \log(2/\delta)}
    \end{align*}
\end{lemma}
Furthermore, \cite{russo_eluder_nodate} show that the Eluder dimension of the linear function class is $\edim(\F_{\mathrm{Linear}}, \epsilon) = \mathcal O (d \log(1/\epsilon))$. Combining these facts with our main regret bound, we have:
\begin{corollary}\label{cor:ex:linearalgorithmbound} 
In the case of linear losses and linear constraints, for any $\delta \in (0,1)$, with probability at least $1-3\delta$ Algorithm \ref{alg:su:gen} satisfies:
\begin{align*}
    \regret_T = \mathcal O \left( \frac{d}{b}\sqrt{T} \log(\frac{T}{d\delta})\right) 
\end{align*}
\end{corollary}	
\begin{remark}
Suppose that $\Action_0$ is the $\ell_1$ ball of diameter $b$ instead of the $\ell_2$ ball of diameter $b$. Then $\F_0$ becomes the unit $\ell_\infty$ ball - and consequently the eluder dimension $\Eluder(\F_0, \epsilon)$ increases by a factor of $\log(d)$. 
\end{remark}

\subsection{Generalized Linear Constraints} 
Let $\sigma : \mathbb R \to [-1,1]$ be a fixed, differentiable non-decreasing link function. Consider the setting of generalized linear constraints where:
$$
    \F_{\mathrm{GL}} = \{ (\x, \cont) \mapsto \sigma \left( \langle \f , \x \rangle - b \right) | \f \in \mathbb R^d, \| \f \| \leq 1 \}
$$
and the unknown constraint is $\sigma(\langle \f^*, \x \rangle - b ) \leq 0$. Assume $\sigma: \reals \mapsto [-1,1]$ is a differentiable, non-decreasing function such that $\sigma(0) = 0$. Further, define $\underline c = \min_{\beta \in [-1,1]} \sigma'(\beta)$ and let $\overline c = \max_{\beta \in [-1,1]} \sigma'(\beta)$.  Let $r := \frac{\overline c}{\underline c}$. We show using the scaling-based mapping introduced in the previous subsection in Algorithm \ref{alg:su:gen} allows us to bound $\kappa^*$ in this setting. 
\begin{lemma}\label{lem:ex:glmrelaxation}
Suppose $\F = \F_{\mathrm{GL}}$ in Algorithm \ref{alg:su:gen} and suppose assumption \ref{as:ex:lipschitz} holds. Let \\ $\gamma_t(\tilde \action_t):= \max\left\{ \gamma \in [0,1] : \gamma \tilde \action_t \in \Pessimistic_t \right\} $, and sample $\action_t \sim  \Map_t(\tilde p_t, \Pessimistic_t, \F_{t})$ by drawing $\tilde \action_t \sim \tilde p_t$ then outputting $\gamma_t(\tilde \action_t)\tilde \action_t$. Then, $\kappa^* \leq \frac{rD_\ell D_\x}{b \underline c}$.
\end{lemma}

\section{Extensions}\label{sec:extensions}
\subsection{Multiple Linear Constraints and Vector Feedback}\label{subsec:multipleconstraints} 
We can extend our algorithms to handle the case of a polytopic constraint:
\begin{align*}
    \F_{\mathrm{Polytopic}} = \{(a,x)  \mapsto Fa - b\vec{1}| F \in \reals^{d \times m}, \forall i \in [m], \|F_i\| \leq 1, b \in \reals\} 
\end{align*}
where $F_i$ denotes the $i$th row of $F$, and $\vec{1}$ is the vector of all ones.  The unknown constraint would be $F^* a - b \vec{1} \leq 0$, and suppose at every time-step we receive $m$-dimensional feedback vector given by  $F^*\x_t + \eg_t$ where in this case the noise vector is simply a standard multivariate normal. In this case, we could invoke $m$ instances of the regression oracle $\oraclesq$ (using simple squared loss regression oracle), maintaining estimates for each of the rows of $F^*$ in parallel. Consequently, we would maintain $m$ separate sets of version spaces, pessimistic sets and optimistic sets per time step, $\{\F_{t,i}, \Pessimistic_{t,i}, \Optimistic_{t,i} \}_{i=1}^m$. Let $\Pessimistic_t = \cap_{i=1}^m \Pessimistic_{t,i}$ and $\Optimistic_t = \cap_{i=1}^m \Optimistic_{t,i}$, and define $\Delta_{\F_t}(\action) := \max_{i \in [m]}\Delta_{\F_{t,i}}(\action)$. 
\begin{lemma}\label{lem:ex:multilinearrelaxation}
Suppose $\F = \F_{\mathrm{Polytopic}}$ in Algorithm \ref{alg:su:gen}. Let $\gamma_t(\tilde \action_t):= \max\left\{ \gamma \in [0,1] : \gamma \tilde \action_t \in \Pessimistic_t \right\} $, and sample $\action_t \sim  \Map_t(\tilde p_t, \Pessimistic_t, \F_{t})$ by drawing $\tilde \action_t \sim \tilde p_t$ then outputting $\gamma_t(\tilde \action_t)\tilde \action_t$. Then, $\kappa^* \leq \frac{D_\ell D_\x}{b}$.
\end{lemma}
However, the eluder dimension of the polytopic function class is $m$ times that of the linear function class:
\begin{corollary}\label{cor:ex:linearalgorithmbound} 
In the case of linear losses and polytopic constraints, for any $\delta \in (0,1)$, with probability at least $1-3\delta$ Algorithm \ref{alg:su:gen} satisfies: 
\begin{align*}
    \regret_T = \bigO \left( \frac{md}{b}\sqrt{T} \log(\frac{T}{d\delta})\right) 
\end{align*}
\end{corollary}

\subsection{Multiple General Constraints with Scalar Feedback}
Suppose as in the previous section we have $m$ constraint functions, $\f^*_1 \ldots \f^*_m$ (these could be linear or more generally from $\F^m$). However, say we don't receive vector feedback of noisy values of each of $\f^*_1 \ldots \f^*_m$ but simply a feedback that is a choice in $[m]$ of a pick of one of the $m$ constraints drawn as per the Boltzman distribution. That is, we draw a choice of one of the $m$ constraints such that probaility of drawing constraint $i$ is proportional to $\exp(\f^\star_i(\action_t,\cont_t))$. This is a natural form of feedback model that says that we are very likely to pick the constraint that is most violated. 
In this case, if we use for online regression oracle the logisitic regression oracle w.r.t. class $\mathcal F^m$, with just this limited feedback obtain regret bounds for safe online learning. 

\section{Discussion and Future Work}
We presented a general safe online learning algorithm that can perform well while strictly adhering to an unknown safety constraint at every time step. We also introduce a complexity measure that captures the per-timestep trade-off between regret minimization and information gain, show asymptotically that this complexity measure is necessary for safe learning, and bound this complexity measure for finite action spaces, linear and generalized linear constraints. 

In terms of future work, we would like to develop constructive and practical algorithms for settings beyond linear optimization with linear constraints. Furthermore, we would like to extend our results to settings with bandit feedback, and to stateful / reinforcement learning settings. 

\paragraph{Acknowledgements} KS acknowledges support from NSF CAREER Award 1750575, and LinkedIn-Cornell grant.

\bibliography{bibliography}
\newpage
\appendix
\section{Proofs from Section \ref{sec:setup}: Setup}
\begin{definition}
    A function $\Link : [-1,1] \to \reals$ is $\lambda$-strongly convex if for all $z, z' \in [-1,1]$, it satisfies 
    \begin{align*}
        \frac{\lambda}{2}(z'-z)^2  \leq \Link(z') - \Link(z) + \link(z)(z-z')s
    \end{align*}
    where $\link(\cdot)$ is the derivative of $\Link$.
\end{definition}
\begin{definition}
    For a link function $\link$ that is the derivative of a $\lambda$-strongly convex function $\Link$, we define the associated loss:
    \begin{align*}
        \lell(z,z') := \Link(z) - \frac{z(z'+1)}{2}
    \end{align*}
\end{definition}
\begin{assumption}[Online Regression Oracle, Regret Version]\label{as:ap:oraclesqregretversion}
The algorithm $\oraclesq$ guarantees that for any (possibly adversarially chosen) sequence $\{\action_t,\cont_t\}_{t=1}^T$ generates predictions $\{\hzf_t\}_{t=1}^T$ satisfying:
$$
\sum_{t=1}^T \lell(\hzf_t - \zf_t) - \inf_{\f \in \f} \sum_{t=1}^T\lell(\f(\action_t,\cont_t) - \zf_t)\le \Regsq^\link (T,\F)
$$
where $\zf_t \sim \link(\f^*(\action_t,\cont_t))$.
\end{assumption}
The following lemma is adapted from \cite{sekhari_contextual_2023}, Lemma 9 and related to \cite{agarwal_selective_nodate}, Lemma 2
\begin{lemma}\label{lem:ap:regrettosquareloss}
    Suppose that $\zf_t$ is generated with a link function $\link$ that is $\lambda$-strongly convex. Suppose that the regression oracle satisfies assumption \ref{as:ap:oraclesqregretversion}. Then for any $\delta \in (0,1)$ and $T \geq 3$, with probability at least $1-\delta$, the regression oracle satisfies assumption \ref{as:su:oraclesq} with:
    \begin{align*}
        \Regsq(T, \delta, \F) \leq \frac{4}{\lambda}\Regsq^\link (T,\F)+\frac{16+24 \lambda}{\lambda^2} \log \left(4 \delta^{-1} \log (T)\right) .
    \end{align*}
\end{lemma}
\begin{proof}
    The proof is an application of \cite{sekhari_contextual_2023} Lemma 9. 
\end{proof}

\begin{proposition*}[Proposition \ref{prop:su:nonconstructive_existence} restated]
There exists an algorithm satisfying Assumption \ref{as:su:oracleol} with 
$$
\regretol(T,\delta) \le 2\ \mathrm{Rad}^{\mathrm{seq}}_{\ell \circ \X}(T) 
$$ 
\end{proposition*}
\begin{proof}
We show that expected regret is bounded by $2\mathrm{Rad}^{\mathrm{seq}}_{\ell \circ \X}(T) $ through a minimax analysis with sequential symmetrization techniques that are now standard from \cite{RST10}. We use the notation $\left<\mathrm{Operator}_t\right>_{t=1}^T[A]$ to denote
$\mathrm{Operator}_1\{\mathrm{Operator}_2\{\ldots\mathrm{Operator}_T\{A\}\ldots\}\}$. We view our online learning setting as a repeated game between adversary and learner where on each round $t$ adversary picks a context and a set $\Action_t$ learner picks a (randomized) action from this set and finally adversary picks $\out_t$ for that round. The value of this game can we written as:
\begin{align*}
\mathrm{Val}_T &= \left<\sup_{\cont_t, \Action_t} \inf_{p_t \in \Simplex(\Action_t)} \sup_{\out_t \in \Out} \mathbb{E}_{\action_t \sim p_t}\right>_{t=1}^T\left[\sum_{t=1}^T \ell(\action_t,\cont_t,\out_t ) - \min_{\action \in \cap_{t=1}^T \Action_t} \sum_{t=1}^T \ell(\action,\cont_t,\out_t )\right]\\
&= \left<\sup_{\cont_t, \Action_t} \sup_{q_t \in \Simplex(\Out)}\inf_{\action_t \in \Action_t} \mathbb{E}_{\out_t \sim q_t}\right>_{t=1}^T\left[\sum_{t=1}^T \ell(\action_t,\cont_t,\out_t ) - \min_{\action \in \cap_{t=1}^T \Action_t} \sum_{t=1}^T \ell(\action,\cont_t,\out_t )\right]\\
&= \left<\sup_{\cont_t, \Action_t} \sup_{q_t \in \Simplex(\Out)} \mathbb{E}_{\out_t \sim q_t}\right>_{t=1}^T\left[\sum_{t=1}^T \inf_{\action_t \in \Action_t} \mathbb{E}_{\out_t \sim q_t} \left[\ell(\action_t,\cont_t,\out_t )\right] - \min_{\action \in \cap_{t=1}^T \Action_t} \sum_{t=1}^T \ell(\action,\cont_t,\out_t )\right]\\
&\le \left<\sup_{\cont_t, \Action_t} \sup_{q_t \in \Simplex(\Out)} \mathbb{E}_{\out_t \sim q_t}\right>_{t=1}^T\left[\sup_{\action \in \cap_{t=1}^T \Action_t}  \left\{\sum_{t=1}^T  \mathbb{E}_{\out_t \sim q_t} \left[\ell(\action,\cont_t,\out_t)\right] -  \ell(\action,\cont_t,\out_t)\right\}\right]\\
&\le \left<\sup_{\cont_t, \Action_t} \sup_{q_t \in \Simplex(\Out)} \mathbb{E}_{\out_t, \out'_t \sim q_t}\right>_{t=1}^T\left[\sup_{\action \in \cap_{t=1}^T \Action_t}  \left\{\sum_{t=1}^T \left( \ell(\action,\cont_t,\out'_t ) -  \ell(\action,\cont_t,\out_t)\right)\right\}\right]\\
&= \left<\sup_{\cont_t, \Action_t} \sup_{q_t \in \Simplex(\Out)} \mathbb{E}_{\out_t, \out'_t \sim q_t} \mathbb{E}_{\epsilon_t}\right>_{t=1}^T\left[\sup_{\action \in \cap_{t=1}^T \Action_t}  \left\{\sum_{t=1}^T \epsilon_t \left( \ell(\action,\cont_t,\out'_t ) -  \ell(\action,\cont_t,\out_t)\right)\right\}\right]\\
&\le \left<\sup_{\cont_t, \Action_t}  \sup_{\out_t, \out'_t \in \Out} \mathbb{E}_{\epsilon_t}\right>_{t=1}^T\left[\sup_{\action \in \cap_{t=1}^T \Action_t}  \left\{\sum_{t=1}^T \epsilon_t \left( \ell(\action,\cont_t,\out'_t ) -  \ell(\action,\cont_t,\out_t)\right)\right\}\right]\\
& \le \left<\sup_{\cont_t, \Action_t}  \sup_{\out_t, \out'_t \in \Out} \mathbb{E}_{\epsilon_t}\right>_{t=1}^T\left[\sup_{\action \in  \Action}  \left\{\sum_{t=1}^T \epsilon_t \left( \ell(\action,\cont_t,\out'_t ) -  \ell(\action,\cont_t,\out_t)\right)\right\}\right]\\
\end{align*}
where first line is obtained using repeated application of minimax theorem (which holds with minor assumptions on action sets and context set etc. that can be found in \cite{RST10}). Second line is a rearrangement. The next line is by noting that infimum over $\Action_t$ can be replaced by a single choice of any action in the intersection set. The rest of the steps above are standard sequential symmetrization arguments.  The key step is the last inequality above where we simply move to upper bound by replacing the intersection by $\Action$ which is a larger set. But once this is done, the inner terms are devoid of $\Action_t$'s and so we drop them in the supremums and this results in the twice, sequential Rademacher complexity of the loss class as required to yield:
$$
\mathrm{Val}_T \le 2\ \mathrm{Rad}^{\mathrm{seq}}_{\ell \circ \X}(T) 
$$
Since value is bounded, there exists a regret minimizing algorithm with required bound and thus concludes the proof.
\end{proof}
\section{Proofs from Section \ref{sec:main_results}: Main Results}
\subsection{Proofs of Upper Bounds}


\begin{lemma}\label{lem:ap:confidenceinterval}
    With probability at least $1-\delta$, for all $t \in [T]$, $\f^* \in \F_t$.
\end{lemma}
\begin{proof}
This follows immediately from Assumption \ref{as:su:oraclesq} which guarantees with probability at least $1-\delta$. 
\begin{align*}
      \sum_{s=1}^T (\f^*(\x_s, \cont_s) - \hzf_s )^2 \leq \Regsq (T,\delta, \F_0) 
\end{align*}
and hence for any $t \in [T]$, 
\begin{align*}
      \sum_{s=1}^{t-1} (\f^*(\x_s, \cont_s) - \hzf_s )^2 \leq \Regsq (T,\delta, \F_0) 
\end{align*}
which shows $\f^* \in \F_t$. 
\end{proof}

\begin{proposition*}[Proposition \ref{prop:mr:optimism} restated] 
Let $\{\Optimistic_t\}_{t=1}^T$ be the sequence of optimistic sets generated by Algorithm \ref{alg:su:gen}. For any $t \in [T]$ and any $\delta \in (0,1)$ with probability at least $(1-\delta) $ we have,
$
\{\action \in \Action: f^\star(\action,\cont_t) \le 0\} \subseteq \Optimistic_t$
\end{proposition*}
\begin{proof}
    By Lemma \ref{lem:ap:confidenceinterval}, we have with probability at least $1-\delta$ that $f^* \in \F_t$ simultaneously for all $t \in [T]$. Take some arbitrary $t \in [T]$. Suppose $\action$ is such that $\f^*(\action,\cont_t) \leq 0$. Then:
    \begin{align*}
          \min_{\f \in \F_t} \f(\action, \cont_t) \leq f^*(\action, \cont_t) \leq 0
    \end{align*}
    This shows that $\action \in \Optimistic_t$  hence $\{\action \in \Action: f^\star(\action,\cont_t) \le 0\} \subseteq \Optimistic_t$
\end{proof}

\begin{proposition*}[Proposition \ref{prop:mr:pessimism} restated]
Let $\{\Pessimistic_t\}_{t=1}^T$ be the sequence of optimistic sets generated by Algorithm \ref{alg:su:gen}. For any $t \in [T]$ and $\action \in \Pessimistic_t$ and for any $\delta \in (0,1)$ with probability at least $(1-\delta) $ we have constraint satisfaction, i.e. $f^\star(\action,\cont_t) \le 0$.
\end{proposition*}
\begin{proof}
    By Lemma \ref{lem:ap:confidenceinterval}, we have with probability at least $1-\delta$ that $f^* \in \F_t$ simultaneously for all $t \in [T]$. Take some arbitrary $t \in [T]$. By the definition of $\action \in \Pessimistic_t$, we have $\max_{\f \in \F} \f(\action,\cont_t) \leq 0$. Then:
    \begin{align*}
          f^*(\action, \cont_t) \leq \max_{\f \in \F_t}  f(\action, \cont_t) \leq 0
    \end{align*}
    Hence, we have constraint satisfaction. 
\end{proof}

The following lemma bounds the number of times the width of the set $\F_t$ can exceed some threshold, and is a variant of Proposition 3 of \cite{russo_eluder_nodate}. It is slightly different as our $\F_t$ are constructed around the predictions produced by $\oraclesq$. We state it for completeness. 
\begin{lemma}\label{lem:ap:violationbound}
Let the sequence $\{\F_t , \action_t, \hat \zf_t \}_{t=1}^T$ be generated by Algorithm \ref{alg:su:gen}. Then, for any sequence of adversarial contexts $\{x_t\}_{t=1}^T$, and $\epsilon > 0$, it holds that
\begin{align*}
    \sum_{t=1}^{T} \ind{\Delta_{\F_t}(\action_t, \cont_t) > \epsilon} \leq \left( \frac{4  \Regsq (T,\delta, \F_0)}{\epsilon^2} + 1 \right) \edim(\F_0, \epsilon)
\end{align*}

\end{lemma}

\begin{proof}

First we claim that for $t \in [T]$ if $\Delta_{\F_t}(\action_t, \cont_t) \geq \epsilon$, then $(\action_t, \cont_t)$ must be $\epsilon$-dependent on at most $\frac{20\Regsq (T,\delta, \F_0)}{\epsilon^2}$ disjoint subsequences of $(\action_1, \cont_1) \cdots (\action_{t-1}, \cont_{t-1})$. Since $\Delta_{\F_t}(\action_t, \cont_t) > \epsilon$, there must exist two functions $\f, \f' \in \F_t$ satisfying $f(\action_t, \cont_t) - f'(\action_t, \cont_t) > \epsilon$. By the definition of $\epsilon$-dependence, if $(\action_t, \cont_t)$ is $\epsilon$-dependent on a sequence $(\action_{i_1}, \cont_{i_1}) \cdots (\action_{i_\tau}, \cont_{i_\tau})$ of its predecessors, we must have $\sum_{j=1}^\tau(f(\action_{i_j}, \cont_{i_j}) - f'(\action_{i_j}, \cont_{i_j}))^2 > \epsilon^2$. Therefore, if $(\action_t, \cont_t)$ is $\epsilon$-dependent on $N$ such subsequences it follows that $\sum_{j=1}^{t-1} (f(\action_{j}, \cont_{j}) - f'(\action_{j}, \cont_{j}))^2 > N\epsilon^2$. Therefore:
\begin{align*}
    N\epsilon^2 &<  \sum_{j=1}^{t-1} (f(\action_{j}, \cont_{j}) - f'(\action_{j}, \cont_{j}))^2 \\
    &=  \sum_{j=1}^{t-1} (f(\action_{j}, \cont_{j}) - \hzf_j + \hzf_j - f'(\action_{j}, \cont_{j}))^2 \\
    &\leq 2 \sum_{j=1}^{t-1} (f(\action_{j}, \cont_{j}) - \hzf_j)^2  +  2 \sum_{j=1}^{t-1}  (f'(\action_{j}, \cont_{j}) - \hzf_j )^2 \\
    &\leq 20\Regsq (T,\delta, \F_0)
\end{align*}
where the first inequality follows from cauchy schwarz on each of the summands, and the second follows from $f,f' \in \F_t$. \\

Second, we claim that for any $k \in [T]$ and any sequence $(\action_1,\cont_1) \cdots (\action_k, \cont_k)$, there must be a $j \leq k$ such that $(\action_j, \cont_j)$ is $\epsilon$-dependent on at least $N = \ceil{k/\edim(\F_0, \epsilon) -1}$ disjoint subsequences of its predecessors. We will show an iterative process of finding such an index $j$. Let $S_1 \cdots S_N$ be $N$ subsequences initialized as $S_i = \{ (\action_i, \cont_i) \}$ for $i \in [N]$. For $j \in [N+1 , k]$ first check if $( \action_j , \cont_j) $ is $\epsilon$-dependent of all $S_i$, $i \in [N]$. If it is, we have found the index $j$ satisfying our condition. Otherwise, pick a $S_i$ such that $x_j$ is $\epsilon$-independent of $S_j$, and add $x_j$ to that $S_i$. By the definition of eluder dimension, the maximum size of each $S_i$, $i \in [N]$ is $\edim(\F_0, \epsilon)$, and because $N * \edim(\F_0, \epsilon) \leq k-1$, the process will terminate. 

Now, let $(\action_{i_1}, \cont_{i_1}) \cdots (\action_{i_k}, \cont_{i_k})$ be the subsequence such that for $j \in [k]$, $\Delta_{\F_t}(\action_{i_j}, \cont_{i_j}) > \epsilon$. By the first claim, each element of this subsequence is $\epsilon$-dependent on at most $\frac{20\Regsq (T,\delta, \F_0)}{\epsilon^2}$ disjoint subsequences. By the second claim, there is some element that is $\epsilon$-dependent on at least $\floor{(k-1)/\edim(\F_0, \epsilon)}$ disjoint subsequences. It follows that $\ceil{(k/\edim(\F_0, \epsilon) - 1} \leq \frac{20\Regsq (T,\delta, \F_0)}{\epsilon^2}$, and hence $k \leq \left( \frac{20\Regsq (T,\delta, \F_0)}{\epsilon^2} + 1 \right)\edim(\F_0, \epsilon) $
\end{proof}
The following Lemma utilizes Lemma \ref{lem:ap:violationbound} to upper bound the sum of $\Delta_{\F_t}$. It is similar in spirit to Lemma 2 of \cite{russo_eluder_nodate}, but our analysis is different and captures a trade-off between $T$ and $\Regsq (T,\delta, \F_0) \edim(\F_0,\cdot)$.

\begin{lemma}\label{lem:ap:sumofwidths}
Let the sequence $\{\F_t, p_t\}_{t=1}^T$ be generated by Algorithm \ref{alg:su:gen}. Then, for any sequence of adversarial contexts $\{x_t\}_{t=1}^T$,  
\begin{align*}
    \sum_{t=1}^T \Ex{\action_t \sim p_t}{\Delta_{\F_t}(\action_t, \cont_t)} \leq \inf_{\alpha}\left\{\alpha T + \frac{20 \Regsq (T,\delta, \F_0) \edim(\F_0,\alpha)}{\alpha} \right\}
\end{align*}
\end{lemma}
\begin{proof}
For a run of Algorithm \ref{alg:su:gen}, let $\{\action_t\}_{t=1}^T$ be any sequence of actions drawn $\action_t \sim p_t$ for all $t \in [T]$. Furthermore, to simplify the notation, let us denote $\Delta_t :=  \Delta_{\F_t}(\action_t, \cont_t)$. Let us consider some arbitrary $\alpha > 0$. Then, for this sequence of actions and contexts,
\begin{align*}
    \sum_{t=1}^T \Delta_{\F_t}(\action_t, \cont_t) &:= \sum_{t=1}^T \Delta_t \\
    &\stackrel{(i)}{=} \sum_{t : \Delta_t \leq \alpha } \Delta_t + \sum_{i=0}^{\log(2/\alpha) - 1} \left( \sum_{t :  2^i \alpha  < \Delta_t \leq 2^{i+1} \alpha} \Delta_t \right)  \\
    &\leq \alpha T + \sum_{i=0}^{\log(2/\alpha) - 1} \left( \sum_{t :  2^i \alpha  < \Delta_t \leq 2^{i+1} \alpha} 2^{i+1} \alpha  \right) \\
    &\stackrel{(ii)}{\le}  \alpha T + \sum_{i=0}^{\log(2/\alpha) - 1} \left(  2^{i+1} \alpha \left( \frac{4 \Regsq (T,\delta, \F_0)}{ 2^{2i} \alpha^2} + 1 \right) \edim(\F_0, 2^i \alpha) \right) \\
    &\stackrel{(iii)}{\le} \alpha T + \sum_{i=0}^{\log(2/\alpha) - 1} \left(  2^{i+1} \alpha \left( \frac{5 \Regsq (T,\delta, \F_0)}{ 2^{2i} \alpha^2}  \right) \edim(\F_0,  2^i \alpha) \right) \\
    &\leq \alpha T + \sum_{i=0}^{\log(2/\alpha) - 1} \left( \frac{10 \Regsq (T,\delta, \F_0)}{ 2^{i} \alpha}  \right) \edim(\F_0,  2^i \alpha) \\
    &\stackrel{(iv)}{\le} \alpha T +  \edim(\F,  \alpha) \sum_{i=0}^{\infty} \frac{10 \Regsq (T,\delta, \F_0)}{ 2^{i} \alpha}   \\
    & \stackrel{(v)}{\le} \alpha T + \frac{20 \Regsq (T,\delta, \F_0)}{ \alpha}  \edim(\F_0, \alpha) \\
\end{align*}
In (i) we set the upper bound to the sum as $\log(2/\alpha) - 1$ since all functions $\f \in \F$ map to $[-1,1]$, hence $\Delta_t \leq 2$ so it is enough to consider $i : 2^{i+1}\alpha \leq  2$ and (ii) follows from Lemma \ref{lem:ap:violationbound}, (iii) follows from the fact that $1 \leq \frac{\Regsq (T,\delta, \F)}{(2^i \alpha)^2}$ for $i \in [\log(2/\alpha) - 1]$ if $T>1$, (iv) follows from the fact that $\edim(\F, \cdot )$ is nonincreasing in its second argument, and (v) is an upper bound from the sum of an infinite series. Therefore, for any sequence $\{\F_t, \action_t, \cont_t\}_{t=1}^T$ generated by Algorithm \ref{alg:su:gen} we have 
\begin{align*}
	\sum_{t=1}^T \Delta_{\F_t}(\action_t, \cont_t) \leq  \alpha T + \frac{20 \Regsq (T,\delta, \F_0)}{ \alpha}  \edim(\F_0, \alpha)
\end{align*}
Now, since this holds for any sequence $\{\F_t, \action_t\}_{t=1}^T$ generated by the algorithm and adversarial contexts $\{x_t\}_{t=1}^T$, it holds in expectation over the algorithm's draws. 
\end{proof}

\begin{theorem*}[Theorem \ref{thm:main:maintheorem} restated]
For any $\delta \in (0,1)$ with probability at least $1-3\delta$, Algorithm \ref{alg:su:gen} produces a sequence of actions $\{\action_t \}_{t=1}^T$ that are safe, and enjoys the following bound on regret: 
\begin{align*}
\regret_{T} &\le \inf_{\kappa >0}\left\{\sum_{t=1}^T V_\kappa(\tilde p_t; \Pessimistic_t,\F_t, \cont_t) +  \kappa \inf_{\alpha}\left\{\alpha T + \frac{20 \Regsq (T,\delta, \F_0) \mathcal{E}(\F_0,\alpha)}{\alpha} \right\} \right\} \\
&+ \regretol(T,\delta) + \sqrt{2 T \log (\delta^{-1})}
\end{align*}
where, 
\begin{align*}
V_\kappa(\tilde p_t; \Pessimistic_t,\F_t,\cont_t)  &= \sup_{\out \in \Out}\left\{ \Ex{\action_t \sim  \Map(\tilde p_t; \Pessimistic_t, \F_t,\cont_t)}{\ell(\action_t,\cont_t,\out)} - \Ex{\tilde a_t \sim \tilde p_t}{\ell(\tilde \action_t,\cont_t,\out)}\right\} -  \kappa \Ex{\action_t \sim  \Map(\tilde p_t; \Pessimistic_t, \F_t,\cont_t)}{\Delta_{\F_{t}}(\action_t,\cont_t)}
\end{align*}
Further, if we use~~ 
$
\kappa^* = \max_{t \in [T]} \sup_{\cont \in \Context, \tilde p \in \Simplex(\Pessimistic_t), \out \in \Out} \frac{\Ex{\action \sim  \Map(\tilde p; \Pessimistic_t, \F_t,\cont)}{\ell(\action,\cont , \out)} - \Ex{\tilde \action \sim \tilde p}{\ell(\tilde \action,\cont, \out)}}{\Ex{\action \sim  \Map(\tilde p; \Pessimistic_t, \F_t,\cont)}{\Delta_{\F_t}(\action,\cont)}}
$, then in the above, $V_{\kappa^*}(\tilde p_t; \Pessimistic_t,\F_t,\cont_t) \le 0$ and so we can conclude that:
\begin{align*}
\regret_{T} & \le  \kappa^* \inf_{\alpha}\left\{\alpha T + \frac{20 \Regsq (T,\delta, \F_0) \Eluder(\F_0,\alpha)}{\alpha} \right\} +  \regretol(T,\delta)  + \sqrt{T \log (\delta^{-1})}
\end{align*}
\end{theorem*}

\begin{proof}
By Proposition \ref{prop:mr:pessimism}, with probability at least $1-\delta$, if we play actions from $\Pessimistic_t$, we can guarantee the all the constraints are satisfied. On the other hand, to bound the regret of our algorithm w.r.t. the optimal action in hindsight that also satisfies constraint on every round, note that
\begin{align*}
\regret_{T} &= \sum_{t=1}^T \ls(\action_t,\cont_t,\out_t) - \min_{\substack{\action \in \Action: \forall t,\\  \f^\star(\action_t,\cont_t) \le 0}} \sum_{t=1}^T \ls(\action,\cont_t,\out_t)\\
&\le \sum_{t=1}^T \ls(\action_t,\cont_t,\out_t) - \min_{\action \in  \cap_{t=1}^T \Optimistic_t} \sum_{t=1}^T \ls(\action,\cont_t,\out_t)\\
&\le \sum_{t=1}^T\bigl( \ls(\action_t,\cont_t,\out_t) - \Ex{\tilde \action_t \sim p_t}{\ls(\tilde \action_t,\cont_t,\out_t)}\bigr) + \sum_{t=1}^T \Ex{\tilde \action_t \sim p_t}{\ls(\tilde \action_t,\cont_t,\out_t)} - \min_{\action \in  \cap_{t=1}^T \Optimistic_t} \sum_{t=1}^T \ls(\action,\cont_t,\out_t)\\
& \le \sum_{t=1}^T\bigl( \ls(\action_t,\cont_t,\out_t) - \Ex{\tilde \action_t \sim p_t}{\ls(\tilde \action_t,\cont_t,\out_t)}\bigr) + \regretol(T,\delta)\\
& \stackrel{(i)}{\le}  \sum_{t=1}^T\bigl( \Ex{\action_t \sim p_t}{\ls(\action_t,\cont_t,\out_t)} - \Ex{\tilde \action_t \sim \tilde p_t}{\ls(\tilde \action_t,\cont_t,\out_t)}\bigr) + \regretol(T,\delta) + \sqrt{T \log (\delta^{-1})}\\
& \le \inf_{\kappa >0}\left\{\sum_{t=1}^T  V_\kappa(\tilde p_t; \Pessimistic_t,\F_t, \cont_t) + \kappa \sum_{t=1}^{T} \Ex{\action_t \sim p_t}{\Delta_{\F_{t}}(\action_t,\cont_t)}  \right\} +  \regretol(T,\delta) + \sqrt{T \log (\delta^{-1})}
\end{align*}
where (i) is an application of Hoeffding Azuma to bound $\sum_{t=1}^T \ls(\action_t,\cont_t,\out_t) -  \sum_{t=1}^T \Ex{\action_t \sim p_t}{\ls(\action_t,\cont_t,\out_t)}$ and:
\begin{align*}
V_\kappa(\tilde p_t; \Pessimistic_t,\F_t,\cont_t)  &= \sup_{\out \in \Out}\left\{ \Ex{\action_t \sim  \Map(\tilde p_t; \Pessimistic_t, \F_t,\cont_t)}{\ell(\action_t,\cont_t,\out)} - \Ex{\tilde a_t \sim \tilde p_t}{\ell(\tilde \action_t,\cont_t,\out)}\right\} -  \kappa \Ex{\action_t \sim  \Map(\tilde p_t; \Pessimistic_t, \F_t,\cont_t)}{\Delta_{\F_{t}}(\action_t,\cont_t)}
\end{align*}
by Lemma \ref{lem:ap:sumofwidths} we can bound the $\sum_{t=1}^T  \kappa \Ex{\action_t \sim p_t}{\Delta_{\F_{t}}(\action_t,\cont_t)} $ term, hence, 
\begin{align*}
\regret_{T} &\le \inf_{\kappa >0}\left\{\sum_{t=1}^T  V_\kappa(\tilde p_t; \Pessimistic_t,\F_t, \cont_t) +  \kappa \inf_{\alpha}\left\{\alpha T + \frac{20 \Regsq (T,\delta, \F_0) \mathcal{E}(\F_0,\alpha)}{\alpha} \right\} \right\} \\
&+ \regretol(T,\delta) + \sqrt{T \log (\delta^{-1})}
\end{align*}
This concludes the first bound - which holds with probability at least $1-3\delta$ as we take a union bound over the online regression oracle guarantee, the online learning oracle guarantee, and the application of Hoeffding Azuma. To conclude the second part of the statement, we need to show that for 
\begin{align*}
\kappa^* = \max_{t \in [T]} \sup_{\cont \in \Context, \tilde p \in \Simplex(\Pessimistic_t), \out \in \Out} \frac{\Ex{\action \sim  \Map(\tilde p; \Pessimistic_t, \F_t,\cont)}{\ell(\action,\cont , \out)} - \Ex{\tilde \action \sim \tilde p}{\ell(\tilde \action,\cont, \out)}}{\Ex{\action \sim  \Map(\tilde p; \Pessimistic_t, \F_t,\cont)}{\Delta_{\F_t}(\action,\cont)}}
\end{align*}
we have that $V_\kappa(\tilde p_t; \Pessimistic_t,\F_t,\cont_t) \le 0$. To this end, note that 
\begin{align*}
&V_\kappa(\tilde p_t; \Pessimistic_t,\F_t,\cont_t) \\
&=  \sup_{\out \in \Out}\left\{ \Ex{\action_t \sim  \Map(\tilde p_t; \Pessimistic_t, \F_t,\cont_t)}{\ell(\action_t,\cont_t,\out)} - \Ex{\tilde a_t \sim \tilde p_t}{\ell(\tilde \action_t,\cont_t,\out)}\right\} -  \kappa \Ex{\action_t \sim  \Map(\tilde p_t; \Pessimistic_t, \F_t,\cont_t)}{\Delta_{\F_{t}}(\action_t,\cont_t)} \\
&=  \sup_{\out \in \Out}\left\{ \Ex{\action_t \sim  \Map(\tilde p_t; \Pessimistic_t, \F_t,\cont_t)}{\ell(\action_t,\cont_t,\out)} - \Ex{\tilde a_t \sim \tilde p_t}{\ell(\tilde \action_t,\cont_t,\out)}\right\} \\
& ~~~~~- \left(\max_{s \in [T]} \sup_{\cont \in \Context, \tilde p \in \Simplex(\Pessimistic_s), \out \in \Out} \frac{\Ex{\action \sim  \Map(\tilde p; \Pessimistic_s, \F_s,\cont)}{\ell(\action,\cont , \out)} - \Ex{\tilde \action \sim \tilde p}{\ell(\tilde \action,\cont, \out)}}{\Ex{\action \sim  \Map(\tilde p; \Pessimistic_s, \F_s,\cont)}{\Delta_{\F_s}(\action_s,\cont)}}\right) \Ex{\action_t \sim  \Map(\tilde p_t; \Pessimistic_t, \F_t,\cont_t)}{\Delta_{\F_{t}}(\action_t,\cont_t)}\\
&\le  \sup_{\out \in \Out}\left\{ \Ex{\action_t \sim  \Map(\tilde p_t; \Pessimistic_t, \F_t,\cont_t)}{\ell(\action_t,\cont_t,\out)} - \Ex{\tilde a_t \sim \tilde p_t}{\ell(\tilde \action_t,\cont_t,\out)}\right\} \\ \\
& ~~~~~- \left(  \tfrac{\sup_{\out \in \Out}\{\Ex{\action_t \sim  \Map(\tilde p_t; \Pessimistic_t, \F_t,\cont_t)}{\ell(\action_t,\cont_t,\out)} - \Ex{\tilde a_t \sim \tilde p_t}{\ell(\tilde \action_t,\cont_t,\out)}\}}{\Ex{\action_t \sim  \Map(\tilde p_t; \Pessimistic_t, \F_t,\cont_t)}{\Delta_{\F_{t}}(\action_t,\cont_t)}}\right) \Ex{\action_t \sim  \Map(\tilde p_t; \Pessimistic_t, \F_t,\cont_t)}{\Delta_{\F_{t}}(\action_t,\cont_t)}\\
& = 0
\end{align*}
\end{proof}

\begin{lemma*}[Lemma \ref{lem:mr:longterm} restated]
    For any $\delta \in (0,1)$, there exists an algorithm that with probability at least $1-2\delta$ produces a sequence of actions $\{ \action_t \}_{t=1}^T$ that satisfies: 
    \begin{align*}
        \regret_T \leq \regretol(T,\delta) ~\textrm{ and }~\sum_{t=1}^T \f^*(\action_t, \cont_t) \leq  \inf_{\alpha}\left\{\alpha T + \frac{20 \Regsq (T,\delta, \F_0) \edim(\F_0,\alpha)}{\alpha} \right\}
    \end{align*}
\end{lemma*}
\noindent We provide a modified version of Algorithm \ref{alg:su:gen}, stated in Algorithm \ref{alg:ap:longterm}, where we do not maintain an pessimistic set, and directly play the output of the $\oracleol$.We claim that Algorithm \ref{alg:ap:longterm} satisfies the guarantee from Lemma \ref{lem:mr:longterm}.
\begin{algorithm}
\caption{Online Learning with Long Term Constraints}
\label{alg:ap:longterm}
\begin{algorithmic}[1]
\STATE Input:  $\oracleol$, $\oraclesq$, Initial safe set $\Action_0$
\STATE $\F_0 = \{\f \in \F: \forall \action \in \Action_0, \forall \cont \in \Context, \f(\action,\cont) \le 0\}$
\STATE Set $\Regsq (T,\delta, \F) = \mathrm{Reg}_{\mathrm{Sq}}(T,\F) + O(\log(\log T)/\delta)$
\FOR{$t = 1,\cdots,T$}
\STATE {\bf Receive context $\cont_t$}
\STATE $\F_t = \{ \f \in \F_0 : \sum_{s=1}^{t-1} Q_s (\f(\x_s) - \hzf_s )^2 \leq \Regsq (T,\delta, \F) \}$
\STATE $\Optimistic_t = \left\{ \action \in \Action : \min_{f \in \F_t} f(\action,\cont_t) \leq 0 \right\}$ 
\STATE $p_t = {\oracleol}_t(\cont_t,\Optimistic_t)$
\STATE Draw $\action_t \sim p_t$ 
\STATE {\bf Ask for noisy feedback $\zf_t$ }
\STATE Update $\hzf_t = {\oraclesq}_t(\cont_t,\action_t)$
\STATE {\bf Play $\action_t$ and receive $\out_t$}
\ENDFOR
\end{algorithmic}
\end{algorithm}
 
\begin{proof}
By Lemma \ref{lem:ap:confidenceinterval}, we know that with probability at least $1-\delta$, for every $T$ simultaneously, $\f^* \in \F_t$. Now for a given timestep $t \in [T]$ and consider $\action_t$. For this action, let $\underline{\f_t} := \argmin_{\f \in \hat \F_t} \f(\action_t, \cont_t)$. Since $\action_t \in \Optimistic_t$, 
\begin{align*}
\underline{\f_t}(\action_t, \cont_t) &\leq 0 \\
\underline{\f_t}(\action_t, \cont_t) - \underline{\f_t}(\action_t, \cont_t) + \f^*(\action_t, \cont_t) &\leq \Delta_{\F_t}(\action_t, \cont_t) \\
\f^*(\action_t, \cont_t) &\leq \Delta_{\F_t}(\action_t, \cont_t)
\end{align*}
Summing up all terms, we get
\begin{align*}
\sum_{t=1}^T\f^*(\action_t, \cont_t) &\leq \sum_{t=1}^T  \Delta_{\F_t}(\action_t, \cont_t)
\end{align*}
Now, using Lemma \ref{lem:ap:sumofwidths} with each $p_t$ defined as point distributions putting all its mass on $a_t$, we can further bound the above as:
\begin{align*}
\sum_{t=1}^T\f^*(\action_t, \cont_t) &\leq  \inf_{\alpha}\left\{\alpha T + \frac{20 \Regsq (T,\delta, \F_0) \edim(\F_0,\alpha)}{\alpha} \right\}
\end{align*}
Finally, since we are just playing the output of our oracle $\oracleol$, the regret bound is simply 
\begin{align*}
    \regret_T \leq \regretol(T,\delta)
\end{align*}
which holds with probability at least $1-2\delta$ as we apply a union bound over the online regression oracle guarantee and the online learning oracle guarantee\end{proof}
\subsection{Proofs of Lower Bounds}

\begin{lemma*}[Lemma \ref{lem:lb2} restated]
Assume that we have a fixed loss function $\ell:\Action \mapsto \reals$ such that for any $\action \in \Action$ satisfying $\f^\star(\action) > 0$, $\ell(\action) = \min_{\action^* \in \Action: \f^\star(\action^*) \le 0 } \ell(\action^*)$. Furthermore, assume that the eluder dimension of $\F$ at any scale $\epsilon > 0$, (with input space $\Action$) is bounded. If for some $c^* > 0$, $\kappa \ge 0$, and any $\Map$, any regret minimizing oracles $\oracleol$ and $\oraclesq$ (assuming regret in both cases is $\smallO(T)$)
$
\lim_{T\rightarrow \infty} \frac{1}{T} \sum_{t=1}^T  V_\kappa(\tilde p_t; \Pessimistic_t,\F_t) \ge c^*
$
then, there exists a set $\Pessimistic^* \supseteq \Action_0$ with the following properties,
\begin{enumerate}
\item Set $\Pessimistic^*$ satisfies constraints, i.e. $\forall \action \in \Pessimistic^*$, $\f^\star(a) \le 0$
\item Define $\F^* = \{ f : \forall \action \in  \Pessimistic^*,~ f(a) = f^\star(a)\}$. For every action $a \in \Action \setminus  \Pessimistic^*$, $\exists f \in \F^*$ such that $f(a) > 0$.  That is,  $\Pessimistic^*$ cannot be expanded to a larger set guaranteed to satisfy constraint.
\item $\Pessimistic^*$ is such that  $\inf_{\action \in  \Pessimistic^*} \ell(a) - \inf_{\action \in \Action: f^\star(\action) \le 0} \ell(a) \ge c^\star$
\end{enumerate}\end{lemma*}
\begin{proof}
First, since loss is fixed and using the property of the loss assumed, any online learning oracle that minimizes regret would have to return distributions over actions $\tilde{p}_t$'s such that $ \lim_{T} \frac{1}{T} \sum_{t=1}^T \Ex{\tilde \action_t \sim \tilde p_t}{\ell(\tilde{\action}_t)} = \inf_{\action \in \Action}\ell(\action)$.  We have from the premise that for any mapping $\Map$,  $\lim_{T\rightarrow \infty} \frac{1}{T} \sum_{t=1}^T  V_\kappa(\tilde p_t; \Pessimistic_t,\F_t) \ge c^*$. Hence this means that for any mapping giving us distributions $p_t$, we have that 
$$
\lim_{T\to \infty} \frac{1}{T} \sum_{t=1}^T \Ex{\action_t \sim p_t}{\ell(\action_t)} - \inf_{\action \in \Action: \f^\star(\action)\leq 0}\ell(\action) \ge c^*
$$
since $\Ex{\action \sim p_t}{\Delta_{\F_t}(\action_t)} \ge 0$. Further, note that since the loss is fixed, if at some point we are able to find  distribution $p_t$ such that $ \Ex{\action_t \sim p_t}{\ell(\action_t)} - \inf_{\action \in \Action: \f^\star(\action)}\ell(\action) < c^*$ then by returning this distribution we would violate the premise that $\lim_{T\rightarrow \infty} \frac{1}{T} \sum_{t=1}^T  V_\kappa(\tilde p_t; \Pessimistic_t,\F_t) \ge c^*$. Hence we have that for any mapping, and any $t$, $\Ex{\action_t \sim \Map(\tilde p_t; \Pessimistic_t, \F_t)}{\ell(\action_t)} - \inf_{\action \in \Action: \f^\star(\action) \le 0}\ell(\action) \ge c^*$. Since this holds for all mappings, let us consider the following mapping
$$
\Map(\tilde p_t; \Pessimistic_t, \F_t) = \left\{ \begin{array}{ll}
\delta\left(\argmin_{\action \in \Pessimistic_t} \ell(\action)\right) & \textrm{ if }\min_{\action \in \Pessimistic_t} \ell(\action) < \inf_{\action \in \Action: \f^\star(\action) \le 0}\ell(\action) + c^*\\
\delta\left(\argmax_{\action \in \Pessimistic_t} \Delta_{\F_t}(\action)\right) & \textrm{otherwise}
\end{array}\right.
$$
where $\delta(\cdot)$ is the point mass distribution. In the above we assume the argmin and argmax exists otherwise we can do a limiting argument. The above mapping is a valid one since loss is fixed and given. Now note that since we already showed that any valid mapping satisfies $\Ex{\action_t \sim \Map(\tilde p_t; \Pessimistic_t, \F_t)}{\ell(\action_t)} - \inf_{\action \in \Action: \f^\star(\action) \le 0}\ell(\action) \ge c^*$ we can conclude $\min_{\action \in \Pessimistic_t} \ell(\action) \ge  \inf_{\action \in \Action: \f^\star(\action) \le 0} \ell(\action) + c^*$. Now define the set
$$
\Pessimistic^* = \bigcup_{t\ge 1} \Pessimistic_t
$$
Since $\Pessimistic_t$'s are all guaranteed to be safe, we have that $\Pessimistic^*$ is also safe, satisfying property $1$. Second, since for every $t$, $\min_{\action \in \Pessimistic_t} \ell(\action) \ge  \inf_{\action \in \Action: \f^\star(\action) \le 0} \ell(\action) + c^*$ we have that $\inf_{\action \in  \Pessimistic^*} \ell(a) - \inf_{\action \in \Action: f^\star(\action) \le 0} \ell(a) \ge c^\star$.  Thus, $\Pessimistic^*$ satisfies property $3$ as well. Finally, to prove property $2$, we use the assumption that eluder dimension for any scale $\epsilon$ is finite and that the regression oracle ensures that regret is sub-linear.  Specifically, assume that online regression oracle guarantees an anytime regret guarantee of $\phi_{\delta}(t)$ with probability $1 - \delta$ for any $t$ rounds. In this case, using Lemma \ref{lem:ap:violationbound} (with $Q_t = 1$ for all $t$) we have that with probability at least $1 - \delta$, for all $T \ge 1$ and all $\epsilon > 0$,
$$
\sum_{t=1}^T \mathbf{1}\{\Delta_{\F_t}(\action_t) > \epsilon\} \le \left(\frac{4 \phi_\delta(T)}{\epsilon^2} + 1\right) \edim(\F, \epsilon)
$$
for $a_t$'s produced by the above mapping. However, since we are picking $a_t$'s that maximize $\Delta_{\F_t}(a_t)$ on every round and because the $\Pessimistic_t$'s are nested,  the indicators are in descending order. Hence, with probability at least $1 - \delta$, for any $\epsilon \in (0,1]$, let $T_\epsilon$ be the smallest integer such that
$$
\frac{T_\epsilon }{\phi_{\delta}(T_{\epsilon})} > \frac{5 \edim(\F, \epsilon)}{\epsilon^2}~,
$$
This is where the condition that regret bound $\phi_{\delta}(T_{\epsilon})$ is $o(T_\epsilon)$ is needed so that the above yields a valid lower bound on $T_\epsilon$. We have that for any $t$, for every action $\action \in \Pessimistic_t$, $\F_{t+ T_\epsilon}$ is such that $\sup_{f \in \F_{t + T_\epsilon}}|f(\action) - f^\star(\action)| \le \epsilon$.  The reason we take $\F_{t+ T_\epsilon}$ is because $t$ is the first round in which actions in $\Pessimistic_t$ not in earlier sets come into consideration and so we need $T_\epsilon$ more rounds to ensure that for all actions in this set, estimation error is smaller than $\epsilon$. Thus if we consider the set $\bigcap_{t \ge 1}\F_t$, this set corresponds to the set $\F^* = \{ f : \forall \action \in  \Pessimistic^*,~ f(a) = f^\star(a)\}$. Further, by definition of $\Pessimistic_t$'s we have that if there were some action $\action$ such that $\forall f \in \F^*$ $f(\action) \le 0$, then this action would be contained in $\Pessimistic^*$. Thus we conclude that every action not in $\Pessimistic^*$ is such that it evaluates to a positive number for some function $f \in \F^*$. Thus we have shown property $2$ as well. 
\end{proof}

\begin{proposition*}[Proposition \ref{prop:lb1} restated] 
If there exists a set $\Pessimistic^*$ that has the following properties,
\begin{enumerate}
\item Set $\Pessimistic^*$ satisfies constraints, i.e. $\forall \action \in \Pessimistic^*$, $\f^\star(a) \le 0$
\item Define $\F^* = \{ f : \forall \action \in  \Pessimistic^*,  f(a) = f^\star(a)\}$. For every action $a \in \Action \setminus  \Pessimistic^*$, $\exists f \in \F^*$ such that $f(a) > 0$.  That is,  $\Pessimistic^*$ cannot be expanded to a larger set guaranteed to satisfy constraint.
\item $\Pessimistic^*$ is such that  $\inf_{\action \in  \Pessimistic^*} \ell(a) - \inf_{\action \in \Action: f^\star(\action) \le 0} \ell(a) \ge c^\star$
\end{enumerate}
Then, safe learning is impossible, and any learning algorithm that is guaranteed to satisfy constraints on every round (with high probability) has a regret lower bounded by $\regret_T \ge T c^*$.
\end{proposition*}

\begin{proof}
By property 1, we are guaranteed that $\Pessimistic^*$ is safe so we can start any algorithm with initial safe set  $\Action_0 = \Pessimistic^*$. 
Since any safe algorithm must play actions that it can guarantee are safe with high probability, initially any algorithm initialized with $\Pessimistic^*$ has to play from within this set till it can verify some action outside of this set is safe. However by property 3, any action within $\Pessimistic^*$ is at least $c^*$ suboptimal. Any feedback $\zf_t$ we obtain in the process of playing actions $\action_t \in \Pessimistic^*$ would certainly help us evaluate $\f^\star(\action_t)$ more accurately. However, property $2$ implies that even if we were given the values of $\f^\star$ for every action in the set $\Pessimistic^*$, we still would not be able to find another action outside of this set that we can conclude is safe unless we make further assumptions on $\f^\star$. This is because, each probe/feedback by playing action $\action_t$ yields value of $f^\star(\action_t) + \eg_t$. Since the noise $\eg_t$ is a standard normal variable, at best we might be able to learn only $\f^\star(\action)$ for every $\action \in \Pessimistic^*$. However, even if we had this information, the best we could conclude is that $\f^\star$ is one of the functions in $\F^*$. However, property 2 ensures that for every $\action \in \Action \setminus \Pessimistic^*$, there is a function in $f \in \F^*$ that matches the value of $\f^\star$ on on every action in $\Pessimistic^*$ but has $f(\action) >0$. Since we have no information about which $f \in \F^*$ is the true $\f^\star$, no learning algorithm will be able to safely try any action outside of $\Pessimistic^*$ and so any safe learning algorithm will suffer a sub-optimality of at least $c^*$ on every round and thus $\regret_T \ge T c^*$\end{proof}

\section{Proofs from Section \ref{sec:examples}: Examples}
\subsection{Proof of Lemmas \ref{lem:ex:mabrelaxation}, \ref{lem:ex:linearrelaxation} and \ref{lem:ex:glmrelaxation}}

\begin{lemma*}[Lemma \ref{lem:ex:mabrelaxation} restated]
Suppose $\F = \F_{\mathrm{FAS}}$ in Algorithm \ref{alg:su:gen} and suppose assumption \ref{as:ex:mab} holds. Suppose we use the mapping defined in equation \ref{eq:ex:mabmapping}. Then, $\kappa^* \leq \frac{1}{\Delta_0}$. 
\end{lemma*}
\begin{proof}
First suppose $|\F_t| = 1$. Then $\Pessimistic_t = \Optimistic_t$, and hence $m_t = 0$, $\gamma = 0$ and $p_t = \tilde p_t$. Using the definition of $\kappa^*$ as in Theorem \ref{thm:main:maintheorem}:
\begin{align*}
    \kappa^* =  \max_{t \in [T]} \sup_{\cont \in \Context, \tilde p \in \Simplex(\Pessimistic_t), \out \in \Out} \frac{\Ex{\action \sim  \Map(\tilde p; \Pessimistic_t, \F_t,\cont)}{\ell(\action,\cont , \out)} - \Ex{\tilde \action \sim \tilde p}{\ell(\tilde \action,\cont, \out)}}{\Ex{\action \sim  \Map(\tilde p; \Pessimistic_t, \F_t,\cont)}{\Delta_{\F_t}(\action,\cont)}} = 0
     \end{align*}
Now, if $|\F_t|> 1$, then $\gamma = 1$ and 
\begin{align*}
    \kappa^* &=  \max_{t \in [T]} \sup_{\cont \in \Context, \tilde p \in \Simplex(\Pessimistic_t), \out \in \Out} \frac{\Ex{\action \sim  \Map(\tilde p; \Pessimistic_t, \F_t,\cont)}{\ell(\action,\cont , \out)} - \Ex{\tilde \action \sim \tilde p}{\ell(\tilde \action,\cont, \out)}}{\Ex{\action \sim  \Map(\tilde p; \Pessimistic_t, \F_t,\cont)}{\Delta_{\F_t}(\action,\cont)}} \\
    &\stackrel{(i)}{=}  \max_{t \in [T]} \sup_{\tilde p \in \Simplex(\Pessimistic_t), \ell_t \in [0,1]^K} \frac{\langle \ell_t , p_t - \tilde p_t \rangle}{\Ex{\action \sim p_t}{\Delta_{\F_t}(\action)}} \\
    &\leq \max_{t \in [T]} \sup_{\tilde p \in \Simplex(\Pessimistic_t), \ell_t \in [0,1]^K} \frac{\| \ell_t \|_\infty \|p_t - \tilde p_t \|_1}{\gamma \Delta_0} \\
    &\leq \frac{\|\gamma e_{\action_{\Delta}}\|_1 + \| \tilde p_t' - \tilde p_t\|_1}{\gamma \Delta_0}\\
    &\leq \frac{\gamma+ (1-\gamma)2m_t}{\gamma \Delta_0}\\
    &\stackrel{}{=} \frac{1}{\Delta_0}
 \end{align*}
 where (i) follows from the simplifying assumptions we made on the losses and that we do not receive a context.   \end{proof}

\begin{lemma*}[Lemma \ref{lem:ex:linearrelaxation} restated]
Suppose $\F = \F_{\mathrm{Linear}}$ in Algorithm \ref{alg:su:gen}. Let $\gamma_t(\tilde \action_t):= \max\left\{ \gamma \in [0,1] : \gamma \tilde \action_t \in \Pessimistic_t \right\} $, and sample $\action_t \sim  \Map_t(\tilde p_t, \Pessimistic_t, \F_{t})$ by drawing $\tilde \action_t \sim \tilde p_t$ then outputting $\gamma_t(\tilde \action_t)\tilde \action_t$. Then, $\kappa^* \leq \frac{D_\ell D_\x}{b}$.
\end{lemma*}
We first introduce a lemma that lower bounds $\gamma_t(\tilde \action_t)$.

\begin{lemma}\label{lem:ap:scaling}
    Suppose $\F = \F_{\mathrm{Linear}}$ in Algorithm \ref{alg:su:gen}. $\gamma_t(\tilde \action_t):= \max\left\{ \gamma \in [0,1] : \gamma \tilde \action_t \in \Pessimistic_t \right\} $ is lower bounded as:
    \begin{align*}
        \gamma_t(\tilde \action_t) \geq \frac{b}{b+\Delta_{\F_t}(\tilde \x_t)} \\
    \end{align*}
\end{lemma}

\begin{proof}
Let $\{\F_t, \Optimistic_t, \Pessimistic_t \}_{t=1}^T$ be generated by Algorithm \ref{alg:su:gen}.
Fix a $t\in [T]$, and let us consider some $\tilde \x_t \in O_t$ and for this $\tilde \x_t$, define $\underline{\f} := \argmin_{\f \in  \F_t} \f(\tilde \action_t)$, and let $\overline \f$ be any function in $ \F_t$.  From the definition of $O_t$, we have $ \underline f(\tilde \action_t) \leq b$. Then: 
\begin{align*}
	\underline \f(\tilde \action_t) & \leq b \\
	\underline \f(\tilde \action_t) + \overline \f(\tilde \action_t) - \underline \f(\tilde \action_t) &\leq b + \Delta_{\F_t}(\tilde \x_t) \\
	\overline \f(\tilde \action_t) &\leq b + \Delta_{\F_t}(\tilde \x_t) \\
\end{align*}
where the second inequality follows from the definition of $\Delta_{\F_t}(\cdot)$. Let $\alpha = \frac{b}{b+\Delta_{\F_t}(\tilde \x_t)}$. Then:
\begin{align*}
     \overline \f(\tilde \action_t) &\leq b + \Delta_{\F_t}(\tilde \x_t) \\
     \alpha \overline \f(\tilde \action_t) & \leq \alpha (b + \Delta_{\F_t}(\tilde \x_t)) \\
     \overline \f(\alpha\tilde \action_t) &\leq b
\end{align*}
where the last line follows from linearity. Since $\overline \f$ was an arbitrary function in $ \F_t$, this shows that $\alpha\tilde \x_t  \in P_t$. Since we defined $\gamma_t(\tilde \action_t) := \max\left\{ \gamma \in [0,1] : \gamma \tilde \x_t \in P_t \right\}$
$$\gamma_t(\tilde \action_t) \geq \alpha = \frac{b}{b+\Delta_{\F_t}(\tilde \x_t)}$$
\end{proof}
We now prove Lemma \ref{lem:ex:linearrelaxation}.
\begin{proof}
%
Using the definition of $\kappa^*$ as in Theorem \ref{thm:main:maintheorem}:
\begin{align*}
    \kappa^* &=  \max_{t \in [T]} \sup_{\cont \in \Context, \tilde p \in \Simplex(\Pessimistic_t), \out \in \Out} \frac{\Ex{\action \sim  \Map(\tilde p; \Pessimistic_t, \F_t,\cont)}{\ell(\action,\cont , \out)} - \Ex{\tilde \action \sim \tilde p}{\ell(\tilde \action,\cont, \out)}}{\Ex{\action \sim  \Map(\tilde p; \Pessimistic_t, \F_t,\cont)}{\Delta_{\F_t}(\action,\cont)}} \\
    &= \max_{t \in [T]} \sup_{\cont \in \Context, \tilde \action_t \in \Pessimistic_t, \out \in \Out} \frac{\ell(\gamma_t(\tilde \action_t) \tilde \action_t,\cont,y) - \ell(\tilde \action,\cont,y)}{\Delta_{\F_t}(\gamma_t(\tilde \action_t) \tilde \action_t)}\\
    &\stackrel{(i)}{=} \max_{t \in [T]} \sup_{\tilde \action \in \Pessimistic_t} \frac{\sup_{\ell_t : \X \to \mathbb R} \left\{\ell_t(\gamma_t(\tilde \action_t) \tilde \action_t)  - \ell_t(\tilde \x_t) \right\}}{\Delta_{\F_t}(\gamma_t(\tilde \action_t)\tilde \action_t)}\\
    &\leq \max_{t \in [T]}  \sup_{\tilde \action \in \Pessimistic_t} \frac{D_\ell \|\gamma_t(\tilde \action_t) \tilde \action_t - \tilde \action_t\|}{\Delta_{\F_t}(\gamma_t(\tilde \action_t) \tilde\action_t)} \\
    &\leq \max_{t \in [T]} \frac{D_\ell D_\action(1 - \gamma_t(\tilde \action_t))}{\gamma_t(\tilde \action_t)\Delta_{\mathcal F}(\tilde \action_t)}\\
    &\leq \max_{t \in [T]} \frac{D_\ell D_\action(\frac{1}{\gamma_t(\tilde \action_t)} - 1)}{\Delta_{\mathcal F}(\tilde \action_t)}\\
    &\stackrel{(ii)}{\leq} \max_{t \in [T]} \frac{D_\ell D_\action(\frac{b+\Delta_{\F_t}(\tilde \x_t)}{b} - 1)}{\Delta_{\mathcal F}(\tilde \action_t)}\\
    &= \frac{D_\ell D_\action}{b}
\end{align*}
where (i) follows from the simplifying assumptions we made on the losses and that we do not receive a context, and (ii) follows from Lemma \ref{lem:ap:scaling}.

\end{proof}

\begin{lemma*}[Lemma \ref{lem:ex:glmrelaxation} restated]
Suppose $\F = \F_{\mathrm{GL}}$ in Algorithm \ref{alg:su:gen}  and suppose assumption \ref{as:ex:lipschitz} holds. Let $\gamma_t(\tilde \action_t):= \max\left\{ \gamma \in [0,1] : \gamma \tilde \action_t \in \Pessimistic_t \right\} $, and sample $\action_t \sim  \Map_t(\tilde p_t, \Pessimistic_t, \F_{t})$ by drawing $\tilde \action_t \sim \tilde p_t$ then outputting $\gamma_t(\tilde \action_t)\tilde \action_t$. Then, $\kappa^* \leq \frac{rD_\ell D_\x}{b \underline c}$.
\end{lemma*}
We first show a lower bound for $\gamma_t(\tilde \action_t)$ in the case of generalized linear constraints, in a manner similar to \ref{lem:ap:scaling}. 
\begin{lemma}\label{lem:ap:glmscaling}
    Suppose $\F = \F_{\mathrm{GL}}$ in Algorithm \ref{alg:su:gen}. $\gamma_t(\tilde \action_t):= \max\left\{ \gamma \in [0,1] : \gamma \tilde \action_t \in \Pessimistic_t \right\} $ is lower bounded as:
    \begin{align*}
        \gamma_t(\tilde \action_t) \geq \frac{b}{b+\frac{1}{\underline c} \Delta_{\F_t}(\tilde \x_t)} \\
    \end{align*}
\end{lemma}

\begin{proof}
Let $\{\F_t, \Optimistic_t, \Pessimistic_t \}_{t=1}^T$ be generated by Algorithm \ref{alg:su:gen}. Fix a $t\in [T]$, and let us consider some $\tilde \x_t \in O_t$ and for this $\tilde \x_t$, define $\underline{\f} := \argmin_{\f \in  \F_t} \f(\tilde \action_t)$, and let $\overline \f$ be any function in $ \F_t$.  From the definition of $O_t$, we have $ \underline f(\tilde \action_t) \leq 0$. Then: 
\begin{align*}
	\underline \f(\tilde \action_t) & \leq 0 \\
	\underline \f(\tilde \action_t) + \overline \f(\tilde \action_t) - \underline \f(\tilde \action_t) &\leq \Delta_{\F_t}(\tilde \x_t) \\
	\overline \f(\tilde \action_t) &\leq  \Delta_{\F_t}(\tilde \x_t) \\
\end{align*}
where the second inequality follows from the definition of $\Delta_{\F_t}(\cdot)$. $\overline f$ can be written as $\overline \f (\cdot) = \sigma( \langle w, \cdot \rangle - b)$ for some $w \in \reals^d$. Let $\alpha = \frac{b}{b+\frac{1}{\underline c} \Delta_{\F_t}(\tilde \x_t)}$. Then:
\begin{align*}
     \overline \f(\tilde \action_t) &\leq \Delta_{\F_t}(\tilde \x_t) \\
     \sigma( \langle w, \tilde \action_t \rangle - b) &\leq \Delta_{\F_t}(\tilde \x_t)\\
     \underline c \left( \langle w, \tilde \action_t \rangle - b \right)  &\leq \Delta_{\F_t}(\tilde \x_t)\\
      \langle w, \tilde \action_t \rangle &\leq b + \frac{1}{\underline c} \Delta_{\F_t}(\tilde \x_t) \\
      \langle w, \alpha \tilde \action_t \rangle &\leq b\\
      \sigma( \langle w, \alpha \tilde \action_t \rangle - b) &\leq 0\\
      \overline \f(\alpha \tilde \action) &\leq 0
\end{align*}
where the second to last inequality follows from the fact that $\sigma(0) = 0$ and $\sigma$ is an increasing function. Since $\overline \f$ was an arbitrary function in $ \F_t$, this shows that $\alpha\tilde \x_t  \in P_t$. Since we defined $\gamma_t(\tilde \action_t) := \max\left\{ \gamma \in [0,1] : \gamma \tilde \x_t \in P_t \right\}$
$$\gamma_t(\tilde \action_t) \geq \alpha = \frac{b}{b+\frac{1}{\underline c} \Delta_{\F_t}(\tilde \x_t)}$$
\end{proof}
We now prove Lemma \ref{lem:ex:glmrelaxation}.
\begin{proof}
First, we show an lower bound for $\Delta_{\F_{t}}(\alpha \x) $ for some $\x \in \X$ and some constant $\alpha \in (0,1)$. 
\begin{align}
    \Delta_{\F_{t}}(\alpha \x) &:= \sup_{\overline f, \underline f} \sigma(\langle \overline f,  \alpha \x \rangle - b ) - \sigma(\langle \underline f , \alpha\x \rangle - b) \nonumber \\
    & \geq \sup_{\overline f, \underline f} \underline c \left( \langle \overline f,  \alpha \x \rangle - b - \langle \underline f , \alpha\x \rangle + b\right) \nonumber \\
& = \sup_{\overline f, \underline f} \underline c \alpha \left( \langle \overline f,   \x \rangle - \langle \underline f , \x \rangle \right) \nonumber  \\
& = \sup_{\overline f, \underline f} \underline c \alpha \left( \langle \overline f,   \x \rangle -b - \langle \underline f , \x \rangle + b \right) \nonumber  \\
&\geq  \sup_{\overline f, \underline f} \frac{\underline c}{\overline c} \alpha \sigma(\langle \overline f,  \x \rangle - b) - \sigma(\langle \underline f , \x \rangle - b) \nonumber  \\
&\geq \frac{1}{r} \alpha \sup_{\overline f, \underline f} \sigma(\langle \overline f,  \x \rangle - b) - \sigma(\langle \underline f , \x \rangle - b) \nonumber  \\
& \geq \frac{1}{r} \alpha    \Delta_{\F_{t}}( \x) \label{eq:ap:glmscalinglowerbound}
\end{align}
Using the definition of $\kappa^*$ as in Theorem \ref{thm:main:maintheorem}:
\begin{align*}
    \kappa^* &=  \max_{t \in [T]} \sup_{\cont \in \Context, \tilde p \in \Simplex(\Pessimistic_t), \out \in \Out} \frac{\Ex{\action \sim  \Map(\tilde p; \Pessimistic_t, \F_t,\cont)}{\ell(\action,\cont , \out)} - \Ex{\tilde \action \sim \tilde p}{\ell(\tilde \action,\cont, \out)}}{\Ex{\action \sim  \Map(\tilde p; \Pessimistic_t, \F_t,\cont)}{\Delta_{\F_t}(\action,\cont)}} \\
    &= \max_{t \in [T]} \sup_{\cont \in \Context, \tilde \action \in \Pessimistic_t, \out \in \Out} \frac{\ell(\gamma_t(\tilde \action_t)\tilde \action_t,\cont,y) - \ell(\tilde \action,\cont,y)}{\Delta_{\F_t}(\gamma_t(\tilde \action_t)\tilde \action_t)}\\
    &\stackrel{(i)}{=} \max_{t \in [T]} \sup_{\tilde \action \in \Pessimistic_t} \frac{\sup_{\ell_t : \X \to \mathbb R} \left\{\ell_t(\gamma_t(\tilde \action_t)\tilde \action_t)  - \ell_t(\tilde \x_t) \right\}}{\Delta_{\F_t}(\gamma_t(\tilde \action_t)\tilde \action_t)}\\
    &\leq \max_{t \in [T]}  \sup_{\tilde \action \in \Pessimistic_t} \frac{D_\ell \|\gamma_t(\tilde \action_t)\tilde \action_t - \tilde \action_t\|}{\Delta_{\F_t}(\gamma_t(\tilde \action_t)\tilde \action_t)} \\
    &\stackrel{(ii)}{\leq} \max_{t \in [T]} \frac{D_\ell D_\action(1 - \tilde \action_t))}{\frac{\gamma_t(\tilde \action_t)}{r} \Delta_{\mathcal \F_t}(\tilde \action_t)}\\
    &\leq \max_{t \in [T]} \frac{r D_\ell D_\action(\frac{1}{\gamma_t(\tilde \action_t)} - 1)}{\Delta_{\mathcal \F_t}(\tilde \action_t)}\\
    &\stackrel{(ii)}{\leq} \max_{t \in [T]} \frac{rD_\ell D_\action(\frac{b+\frac{1}{\underline c}\Delta_{\F_t}(\tilde \x_t)}{b} - 1)}{\Delta_{\mathcal \F_t}(\tilde \action_t)}\\
    &\leq  \max_{t \in [T]} \frac{rD_\ell D_\action(\frac{1}{\underline c}\Delta_{\F_t}(\tilde \x_t))}{\Delta_{\mathcal \F_t}(\tilde \action_t)}\\
    &= \frac{rD_\ell D_\action}{b \underline c}
\end{align*}
where (i) follows from the simplifying assumptions we made on the losses and that we do not receive a context, and (ii) follows from equation \ref{eq:ap:glmscalinglowerbound}, and (iii) follows from lemma \ref{lem:ap:glmscaling}.

\end{proof}

\subsection{Proof of Lemma \ref{lem:ex:linearconstructiveoracleol}}
We present a constructive online learning oracle for the case of linear cost functions. It is presented in Algorithm \ref{alg:ap:linearconstructiveoracleol}, and it is a projected online gradient descent based algorithm.

\begin{algorithm}
\caption{$\oracleol$ for Linear Losses}
\label{alg:ap:linearconstructiveoracleol}
\begin{algorithmic}[1]
\STATE Input: $\X, D_\x, D_\f, \delta \in (0,1), \eta $
\FOR{timesteps $t = 1,\cdots,T$}
\STATE Receive $\X_t = \Optimistic_t$
\STATE $\bar \x_t \leftarrow \Pi_{\conv( \X_t)} \left( \bar \x_{t-1} - \eta \lif_{t-1}\right)$
\STATE Decompose  $\bar \x_t = \sum_{i=1}^{d+1} p_{t,i} \x_{t,i}$, $\forall i, \x_{t,i} \in  \X_t$
\STATE Sample $\tilde \x_t \sim p_{t}$ 
\STATE Receive $\nabla_t = \lif_{t}$ 
\ENDFOR
\end{algorithmic}
\end{algorithm} 
\noindent We briefly describe the the steps in Algorithm \ref{alg:ap:linearconstructiveoracleol}. \\ 
\noindent \textbf{Convex Hulls of $ \X_t$ (line 4)} \\
Because the action sets $ \X_t = O_t$ are sublevel sets of a minimum of affine functions, they are not necessarily convex, making them incompatible with projection based online learning algorithms. In order to address this, we take the convex hull of $\tilde X_t$, $\conv(\tilde X_t)$, as our projection set. \\
\textbf{Projected Online Gradient Descent (line 4)} \\
Our algorithm then performs projected online gradient descent in sets $\conv(\tilde \X_t)$, generating a sequence of vectors $\{ \bar \x_1 \cdots \bar \x_T \}$ produced by $\bar \x_t = \Pi_{\conv(\tilde \X_t)} \left( \bar \x_{t-1} - \eta \lif_{t-1}\right)$. We note that while we use projected online gradient descent, because the vectors $\{ \bar \x_1 \cdots \bar \x_T \}$ are maintained and updated independently, we could alternatively use a projected variant of any other online convex optimization algorithm that guarantees low regret instead. \\
\textbf{Sampling a Point in $\tilde X_t$ (line 5)} \\
Due to Carath\'eodory's theorem, we know that can write any $\bar x_t \in \conv(\tilde \X_t)$ as a linear combination of at most $d+1$ vectors in $\tilde \X_t$, $\bar \x_t = \sum_{i=1}^{d+1} p_{t,i} \x_{t,i}, \forall i, \x_{t,i} \in \tilde \X_t $. In line 4, we perform this decomposition, and in line 5, we sample the vector $\tilde x_t$ according to this distribution $p_t$. Notably, the point $\tilde \x_t$ satisfies $\mathbb E[\tilde \x_t] = \bar \x_t$.\\

\begin{lemma*}[Lemma \ref{lem:ex:linearconstructiveoracleol} restated]
    $\oracleol$ as presented in Algorithm \ref{alg:ap:linearconstructiveoracleol} satisfies, for any $\delta \in (0,1)$, with probability at least $1-\delta$:
\begin{align*}
   \regretol(T,\delta) \leq 4D_\f D_\x \sqrt{T\log(2/\delta) }  
\end{align*}
\end{lemma*}
\begin{proof}
At every timestep $t \in [T]$, Algorithm $\ref{alg:ap:linearconstructiveoracleol}$ receives a set $\tilde \Action_t = \Optimistic_t$, and produces a $\bar \action_t$ by: 
\begin{align*}
     \bar \action_t \leftarrow \Pi_{\conv(\tilde \Action_t)} \left( \bar \action_{t-1} - \eta \f_{t-1}\right)
\end{align*}
then, it decomposes each $\bar \action_t$ as:
\begin{align*}
    \bar \action_t = \sum_{i=1}^{d+1} p_{t,i} \action_{t,i} \forall i, \action_{t,i} \in \tilde \Action_t
\end{align*}
and then $\tilde \action_t$ is produced by sampling: $\tilde \action_t \sim p_t$. We analyze the regret of Algorithm \ref{alg:ap:linearconstructiveoracleol} by decomposing it into two terms:
\begin{align*}
\regretol(T,\delta) &= \sum_{t=1}^T \langle \f_t , \tilde \action_t \rangle - \min_{\action \in \bigcap_{t=1}^T \action_t} \langle \f_t , \action \rangle \\
&= \underbrace{\sum_{t=1}^T \langle \f_t , \tilde \action_t \rangle - \langle \f_t , \bar \action_t \rangle}_{\text{Term I}} + \underbrace{\sum_{t=1}^T \langle \f_t , \bar \action_t \rangle - \min_{\action \in \bigcap_{t=1}^T \tilde \action_t} \langle \f_t , \action \rangle}_{\text{Term II}}
\end{align*}
\textbf{Bounding Term I} \\
We show that Term I is a difference between a bounded random variable and its expectation, and use Hoeffding's inequality to bound it.  Let  $S_T := \sum_{t=1}^T \langle \f_t, \tilde \action_t \rangle$.  Then:
    \begin{align*}
        \E[S_T] &= \E[ \sum_{t=1}^T \langle \f_t, \tilde \action_t \rangle] \\
        &= \sum_{t=1}^T \langle \f_t , \E [\tilde \action_t] \rangle \\
        &= \sum_{t=1}^T \langle \f_t, \bar \action_t \rangle 
    \end{align*}
    where the second equality follows by linearity of expectation. Note that each summand in $S_T$ satisfies $|\langle  \f_t , \tilde \action_t \rangle| \leq \|\f_t\| \|\tilde \action_t\| \leq D_\f D_\action$. Hence, by Hoeffding's inequality, with probability at least $1-\delta$, 
    \begin{align}\label{eq:ap:ocoalg1}
        \text{Term I} = \sum_{t=1}^T \langle \f_t , \tilde \action_t \rangle - \langle \f_t , \bar \action_t \rangle \leq  |S_T - \mathbb E[S_T]| \leq \sqrt{2TD_\f^2 D_\action^2 \log(2/\delta)}
    \end{align}
\textbf{Bounding Term II} \\
Term II captures the performance of the online gradient descent portion, line 4, of Algorithm \ref{alg:ap:linearconstructiveoracleol}. A difference is that the projection set is time-varying - yet this does not pose a problem for us since we only need to guarantee performance w.r.t. a $a^*$ in the intersection of all the sets. 
Let $\action^* := \argmin_{\action \in \cap_{t=1}^T \tilde \action_t} \langle \f_t, \action \rangle$. with this, 
\begin{align*}
    \text{Term II} = \sum_{t=1}^T \langle \f_t , \bar \action_t \rangle - \min_{\action \in \bigcap_{t=1}^T \tilde \action_t} \langle \f_t , \action \rangle =  \sum_{t=1}^T \langle \f_t , \bar \action_t - \action^*\rangle 
\end{align*}
Therefore, it is sufficient to bound the terms $\langle \f_t, \bar \action_t - \action^* \rangle$. For any timestep $t$, we have:
\begin{align*}
    \left\|\bar \action_{t+1} - \action^*\right\|_2^2 &= \left\| \Pi_{\tilde \action_t}\left(\bar \action_t-\eta \f_t \right) - \action^*\right\|_2^2 \\
    &\leq\left\|\bar \action_t-\eta \f_t - \action^*\right\|_2^2 \\
    &=\left\|\bar \action_t - \action^*\right\|_2^2+\eta^2\left\|\f_t\right\|_2^2-2 \eta\left\langle\f_t, \bar \action_t - \action^*\right\rangle
\end{align*}
where the inequality follows from the fact that $\action^* \in \cap_{t=1}^T \tilde \action_t \subseteq \tilde \action_t$, so projection to $\tilde \action_t$ only decreases the distance. Rearranging, 
\begin{align*}
    \left\langle\f_t, \bar \action_t - \action^*\right\rangle \leq \frac{1}{2 \eta}\left(\left\| \bar \action_t - \action^*\right\|_2^2-\left\| \bar \action_{t+1} - \action^*\right\|_2^2\right)+\frac{\eta}{2}\left\|\f_t\right\|_2^2
\end{align*}
Summing up the terms $t \in [T]$, we get: 
\begin{align*}
\sum_{t=1}^T\left\langle\f_t, \bar \action_t - \action^*\right\rangle & \leq \frac{1}{2 \eta} \sum_{t=1}^T\left(\left\|\bar \action_t - \action^*\right\|_2^2-\left\|\bar \action_{t+1} - \action^*\right\|_2^2\right)+\frac{\eta}{2} \sum_{t=1}^T\left\|\f_t\right\|_2^2 \\
& =\frac{1}{2 \eta}(\left\|\bar \action_1 - \action^*\right\|_2^2-\left\| \bar \action_{T+1} - \action^*\right\|_2^2)+\frac{\eta}{2} T D_\f^2 \\
& \leq \frac{4D_\action^2}{2 \eta} +\frac{\eta}{2} T D_\f^2
\end{align*}
Setting $\eta = \frac{2D_\action}{D_\f\sqrt{T}}$, we get 
\begin{align}\label{eq:ap:ocoalg2}
	\text{Term II} = \sum_{t=1}^T\left\langle\f_t, \bar \action_t - \action^*\right\rangle \leq 2D_\action D_\f \sqrt{T} 
\end{align}
Combining the bounds from equations \ref{eq:ap:ocoalg1} and \ref{eq:ap:ocoalg2}, we get 
\begin{align*}
   \sum_{t=1}^T \langle \f_t , \tilde \action_t \rangle - \min_{\action \in \bigcap_{t=1}^T \action_t} \langle \f_t , \action \rangle = \text{Term I} + \text{Term II} \leq 4D_\action D_\f \sqrt{T\log(2/\delta) }  
\end{align*}
\end{proof}

\section{Proofs from Section \ref{sec:extensions}: Extensions}
\subsection{Multiple Linear Constraints and Vector Feedback}
\begin{lemma*}[Lemma \ref{lem:ex:multilinearrelaxation} restated]
Suppose $\F = \F_{\mathrm{Polytopic}}$ in Algorithm \ref{alg:su:gen}. Let $\gamma_t(\tilde \action_t):= \max\left\{ \gamma \in [0,1] : \gamma \tilde \action_t \in \Pessimistic_t \right\} $, and sample $\action_t \sim  \Map_t(\tilde p_t, \Pessimistic_t, \F_{t})$ by drawing $\tilde \action_t \sim \tilde p_t$ then outputting $\gamma_t(\tilde \action_t)\tilde \action_t$. Then, $\kappa^* \leq \frac{D_\ell D_\x}{b}$.
\end{lemma*}
The proof of this lemma is almost identical to that of Lemma \ref{lem:ex:linearrelaxation}, with the only difference being the fact we need to handle multiple constraints in parallel. We first introduce a lemma that lower bounds $\gamma_t(\tilde \action_t)$.

\begin{lemma}\label{lem:ap:scalingpoly}
    Suppose $\F = \F_{\mathrm{Polytopic}}$ in Algorithm \ref{alg:su:gen}. $\gamma_t(\tilde \action_t):= \max\left\{ \gamma \in [0,1] : \gamma \tilde \action_t \in \Pessimistic_t \right\} $ is lower bounded as:
    \begin{align*}
        \gamma_t(\tilde \action_t) \geq \frac{b}{b+\Delta_{\F_t}(\tilde \x_t)} \\
    \end{align*}
\end{lemma}

\begin{proof}
Let $\{ \{ \F_{t,i}, \Optimistic_{t,i}, \Pessimistic_{t,i} \}_{i=1}^m \}_{t=1}^T$ be generated by the variation of Algorithm \ref{alg:su:gen} described in subsection \ref{subsec:multipleconstraints} that tracks the constraints in parallel.
Fix a $t\in [T]$, and let us consider some $\tilde \x_t \in O_t$ and for this $\tilde \x_t$ and some $i \in [m]$, define $\underline{\f_i} := \argmin_{\f_i \in  \F_{t,i}} \f_i(\tilde \action_t)$, and let $\overline{\f_i}$ be any function in $ \F_{t,i}$.  By definition, $O_t \subseteq O_{t,i}$ so, we have $ \underline {f_i}(\tilde \action_t) \leq b$. Then: 
\begin{align*}
	\underline{\f_i}(\tilde \action_t) & \leq b \\
	\underline{\f_i}(\tilde \action_t) + \overline{\f_i}(\tilde \action_t) - \underline{\f_i}(\tilde \action_t) &\leq b + \Delta_{\F_{t,i}}(\tilde \x_t) \\
	\overline{\f_i}(\tilde \action_t) &\leq b + \Delta_{\F_{t,i}}(\tilde \x_t) \\
\end{align*}
where the second inequality follows from the definition of $\Delta_{\F_{t,i}}(\cdot)$. Let $\alpha_i = \frac{b}{b+\Delta_{\F_{t,i}}(\tilde \x_t)}$. Then:
\begin{align*}
     \overline{\f_i}(\tilde \action_t) &\leq b + \Delta_{\F_{t,i}}(\tilde \x_t) \\
     \alpha_i \overline{\f_i}(\tilde \action_t) & \leq \alpha_i (b +\Delta_{\F_{t,i}}(\tilde \x_t)) \\
     \overline{\f_i}(\alpha_i\tilde \action_t) &\leq b
\end{align*}
where the last line follows from linearity. Since $\overline{\f_i}$ was an arbitrary function in $ \F_{t,i}$, this shows that $\alpha_i\tilde \x_t  \in P_{t,i}$. Therefore, if we set $\alpha = \min_{i\in[m]} \alpha_i$, we have $\alpha \tilde \action_t \in P_t = \cap_{i=1}^m P_{t,i}$.

Since we defined $\gamma_t(\tilde \action_t) := \max\left\{ \gamma \in [0,1] : \gamma \tilde \x_t \in P_t \right\}$
$$\gamma_t(\tilde \action_t) \geq \alpha = \min_{i \in [m]} \left\{ \frac{b}{b+\Delta_{\F_{t,i}}(\tilde \x_t)} \right\} = \frac{b}{b+\max_{i\in [m]}\Delta_{\F_{t,i}}(\tilde \x_t)} =  \frac{b}{b + \Delta_{\F_t}(\tilde \action_t)}$$
\end{proof}
We now prove Lemma \ref{lem:ex:multilinearrelaxation}.
\begin{proof}
%
Using the definition of $\kappa^*$ as in Theorem \ref{thm:main:maintheorem}:
\begin{align*}
    \kappa^* &=  \max_{t \in [T]} \sup_{\cont \in \Context, \tilde p \in \Simplex(\Pessimistic_t), \out \in \Out} \frac{\Ex{\action \sim  \Map(\tilde p; \Pessimistic_t, \F_t,\cont)}{\ell(\action,\cont , \out)} - \Ex{\tilde \action \sim \tilde p}{\ell(\tilde \action,\cont, \out)}}{\Ex{\action \sim  \Map(\tilde p; \Pessimistic_t, \F_t,\cont)}{\Delta_{\F_t}(\action,\cont)}} \\
    &= \max_{t \in [T]} \sup_{\cont \in \Context, \tilde \action_t \in \Pessimistic_t, \out \in \Out} \frac{\ell(\gamma_t(\tilde \action_t) \tilde \action_t,\cont,y) - \ell(\tilde \action,\cont,y)}{\Delta_{\F_t}(\gamma_t(\tilde \action_t) \tilde \action_t)}\\
    &\stackrel{(i)}{=} \max_{t \in [T]} \sup_{\tilde \action \in \Pessimistic_t} \frac{\sup_{\ell_t : \X \to \mathbb R} \left\{\ell_t(\gamma_t(\tilde \action_t) \tilde \action_t)  - \ell_t(\tilde \x_t) \right\}}{\Delta_{\F_t}(\gamma_t(\tilde \action_t)\tilde \action_t)}\\
    &\leq \max_{t \in [T]}  \sup_{\tilde \action \in \Pessimistic_t} \frac{D_\ell \|\gamma_t(\tilde \action_t) \tilde \action_t - \tilde \action_t\|}{\Delta_{\F_t}(\gamma_t(\tilde \action_t) \tilde\action_t)} \\
    &\leq \max_{t \in [T]} \frac{D_\ell D_\action(1 - \gamma_t(\tilde \action_t))}{\gamma_t(\tilde \action_t)\Delta_{\mathcal F}(\tilde \action_t)}\\
    &\leq \max_{t \in [T]} \frac{D_\ell D_\action(\frac{1}{\gamma_t(\tilde \action_t)} - 1)}{\Delta_{\mathcal F}(\tilde \action_t)}\\
    &\stackrel{(ii)}{\leq} \max_{t \in [T]} \frac{D_\ell D_\action(\frac{b+\Delta_{\F_t}(\tilde \x_t)}{b} - 1)}{\Delta_{\mathcal F}(\tilde \action_t)}\\
    &= \frac{D_\ell D_\action}{b}
\end{align*}
where (i) follows from the simplifying assumptions we made on the losses and that we do not receive a context, and (ii) follows from Lemma \ref{lem:ap:scalingpoly}.

\end{proof}

\end{document}